\documentclass[journal,twoside,web]{ieeecolor}
\usepackage{generic}
\usepackage{cite}
\usepackage{amsmath,amssymb,amsfonts}
\usepackage{algorithmic}
\usepackage{graphicx}
\usepackage{textcomp}
\def\BibTeX{{\rm B\kern-.05em{\sc i\kern-.025em b}\kern-.08em
    T\kern-.1667em\lower.7ex\hbox{E}\kern-.125emX}}
\usepackage{booktabs}
\newtheorem{proposition}{Proposition}
\newtheorem{corollary}{Corollary}
\newtheorem{lemma}{Lemma}
\usepackage{colortbl}
\usepackage{multirow}
\usepackage{makecell}
\usepackage{booktabs}
\usepackage{threeparttable}

\definecolor{gray25}{rgb}{0.85, 0.85, 0.85}

\usepackage[colorlinks=true,linkcolor=red,citecolor=blue,urlcolor=blue]{hyperref}

\begin{document}
\title{Predicting Large-scale Urban Network Dynamics with Energy-informed Graph Neural Diffusion}
\author{}
\author{Tong Nie, Jian Sun,~\IEEEmembership{Senior Member,~IEEE,} Wei Ma$^*$,~\IEEEmembership{Member,~IEEE}
\thanks{
This research was sponsored by the National Natural Science Foundation of China (524B2164, 52125208) and the Research Grants Council of the Hong Kong Special Administrative Region, China (Project No. PolyU/25209221, PolyU/15206322, and PolyU/15227424).}
\thanks{Tong Nie and Wei Ma are with the Department of Civil and Environmental Engineering, The Hong Kong Polytechnic University, Hong Kong SAR, China. E-mail: tong.nie@connect.polyu.hk, wei.w.ma@polyu.edu.hk.)}
\thanks{Jian Sun is with the Department of Traffic Engineering and Key Laboratory of Road and Traffic Engineering, Ministry of Education, Tongji University. Shanghai, China. 201804. (E-mail: sunjian@tongji.edu.cn.)
}
\thanks{Corresponding author: Wei Ma.}
}

\maketitle

\begin{abstract}
Networked urban systems facilitate the flow of people, resources, and services, and are essential for economic and social interactions. These systems often involve complex processes with unknown governing rules, observed by sensor-based time series. To aid decision-making in industrial and engineering contexts, data-driven predictive models are used to forecast spatiotemporal dynamics of urban systems. Current models such as graph neural networks have shown promise but face a trade-off between efficacy and efficiency due to computational demands. Hence, their applications in large-scale networks still require further efforts. 
This paper addresses this trade-off challenge by drawing inspiration from physical laws to inform essential model designs that align with fundamental principles and avoid architectural redundancy.
By understanding both micro- and macro-processes, we present a principled interpretable neural diffusion scheme based on Transformer-like structures whose attention layers are induced by low-dimensional embeddings.
The proposed scalable spatiotemporal Transformer (ScaleSTF), with linear complexity, is validated on large-scale urban systems including traffic flow, solar power, and smart meters, showing state-of-the-art performance and remarkable scalability. 
Our results constitute a fresh perspective on the dynamics prediction in large-scale urban networks. 
\end{abstract}


\section{Introduction}

\IEEEPARstart{U}{rban} networks comprise interlinked centers within cities that promote the movement of individuals, resources, and services, thereby fostering economic and social exchanges.
These networks, including transportation systems, production infrastructures, and energy hubs, are governed by complex processes with unknown physical principles.
The direct measure of such unknown dynamics is the sensor-based time series. To help decision-makers obtain accurate and prompt decisions in industrial and engineering applications, data-driven predictive models are established to correlate the observed series and forecast the spatiotemporal evolution of the system.
One key ingredient in modeling the interactive process is the relation among instances. By abstracting instance interactions as graphs, significant progress has been made in developing deep geometric neural architectures to predict dynamics, such as graph neural networks (GNNs) \cite{prasse2022predicting} and Transformers \cite{wu2023difformer}. These models have demonstrated remarkable performances in predicting static graphs \cite{chamberlain2021grand,thorpe2022grand++} and spatiotemporal graphs \cite{li2017diffusion,wu2019graph,bai2020adaptive,cini2022scalable} in urban systems.

The philosophy of these models is to learn meaningful node and graph representations (a.k.a. embeddings) that can effectively leverage collective information from other instances to better predict the dynamics of each individual and uncover latent structures, especially under limited computation budget \cite{cini2022scalable,wu2023difformer,liu2024largest}.
However, a worrying trend that has emerged in current architectural designs is their increasing complexity and difficulty in understanding the mechanism. Due to the lack of physical priors about the data generation process, the stacking of complex modules becomes common practice to meet the requirements of high expressiveness \cite{liu2023we,liu2024largest}. 
These ``black-box'' modules are associated with increased computational burdens and data-hungry architectures, making them difficult to deal with high-dimensional urban networks. Therefore, the \textbf{dilemma arises that current models have to achieve a compromise between effectiveness and efficiency.}

Particularly focused on the urban time series forecasting perspective, the two prevailing lines of research each tend to favor one side of this trade-off. 
First, graph-based methods \cite{shao2022spatial,nie2024channel} reduce the complexity of addressing spatial heterogeneity by introducing learnable node embeddings. The learned inductive bias can alleviate the difficulty in designing complicated models. 
Second, Transformer-based models \cite{STAEformer,jiang2023pdformer} further pursue extreme high performance. The global attention enables them to exploit unobserved interactions and long-range dependencies, surpassing its counterparts, such as GNNs, with high expressivity.
However, the former has limited capacity to learn complex networks and is restricted to small- and medium-scale datasets. The latter uses computationally expensive techniques with potential redundancy that impair their scalability to process large-scale networks under constrained resources. 
More importantly, \textbf{there is a lack of a principled perspective to derive the modeling process and a unified way to inherit the merits of both paradigms.}


In summary, the lack of known driving mechanisms of existing models often necessitates the reliance on stacked black-box modules, which results in high computational overhead to achieve desired accuracy. 
To break this trade-off, we draw inspiration from physical laws and interpret the spatiotemporal process with a general network dynamical model. This allows us to design specific modules that align with fundamental principles that describe the regularity of network dynamics, thereby ensuring accuracy while avoiding redundant design.
{Therefore, we present ScaleSTF, a unified, physically grounded framework that combines an energy‐regularized diffusion interpretation with a Transformer‐style architecture to deliver both competitive spatiotemporal prediction accuracy and linear‐complexity scalability on large urban networks.
}
{As shown in Fig. \ref{fig:overview},} from a micro viewpoint, we elucidate the dynamics with an energy-reduced neural diffusion scheme; from a macro perspective, we connect it with a graph signal denoising process. Our theoretical analysis indicates that for an associated energy measure, there is an equivalence between the discrete diffusion scheme and ultimate states of the graph denoising process. This provides a principled
perspective to inform model designs and we then present an interpretable and expressive neural diffusion scheme based on the Transformer-like structure.
To simultaneously preserve efficiency for large urban networks, we encourage scalability by introducing a low-dimensional embedding method and integrating it into the attentive aggregation of dominant node representations.

In general, the proposed model has linear complexity with respect to the dimension of the network, making it scalable for large urban systems. Empirical evaluations are performed on real-world and synthetic urban datasets, including traffic flow, solar power, and smart meters.
The results show that our model preserves the expressivity of advanced Transformers to achieve state-of-the-art (SOTA) performances and also delivers high computational efficiency.
Our contributions are threefold:
\begin{enumerate}
    \item We theoretically interpret the urban dynamics prediction model by linking the energy-regularized neural diffusion process with a global graph signal denoising problem;
    \item A scalable Transformer-like model called ScaleSTF is developed for large graphs with a low-rank embedding and a modulated node attention in linear complexity;
    \item Large-scale experiments with thousands of nodes show the remarkable scalability and SOTA performance. 
\end{enumerate}

The remainder of this paper is structured as follows. Section \ref{sec:related} briefly reviews related literature. 
Section \ref{sec:motivation} establishes a theoretical analysis of urban network dynamics and presents our motivation. Section \ref{sec:model} elaborates on the proposed model.
Section \ref{sec:exp} performs empirical evaluations using both real-world and simulated urban data.
Section \ref{sec:conclusion} concludes this work and provides future directions. 

\section{Related Work}\label{sec:related}
This section briefly reviews related studies. First, as our work naturally connects with time series (spatiotemporal) forecasting models, we introduce recent advances on data-driven forecasting. Then, we discuss several pioneering works on scalable methods for large networks and show how their methods differ from the present study. Last, graph models based on continuous dynamics are revisited as foundations.
\subsection{Data-driven Time Series Forecasting}
The dynamics of urban networks is usually sensed as time series. The measured time series can be correlated by their physical properties, causing a graph structure. Using this structure and observed data, STGNNs and Transformers are widely developed to predict their short-term variations and long-term periodic behaviors \cite{yu2017spatio, li2018diffusion, wu2019graph, bai2020adaptive, wu2021autoformer,zhou2022fedformer,nie2024imputeformer}.
These data-driven models have shown improved performance compared to traditional statistical methods and been widely applied in various domains of urban studies, such as traffic, energy, meteorology, and environment.
However, deep time series models struggle to comprehend the underlying physical regularities of urban dynamics, forcing practitioners to stack numerous ``black-box'' modules to construct complex architectures. This makes them neither intuitive nor efficient for large-scale applications.

\subsection{Scalable Methods for Large Urban Networks}
Real-world urban networks such as transportation networks are massive in scale. The high dimensionality of variables to be predicted hinders the application of computationally extensive methods. To this end, the focus has recently shifted to developing scalable models for large networks. 
In particular, Cini et al. \cite{cini2022scalable} proposed a scalable graph predictor based on the random walk diffusion operation and the echo state network to encode spatiotemporal representations prior to training. 
Liu et al. \cite{liu2023we} developed two alternative techniques, including a preprocessing-based ego graph and a global sensor embedding to model spatial correlations. The processed spatial features are further fed to temporal models such as RNNs and WaveNets. A graph sampling strategy is required to improve training performance. However, both approaches rely on complex temporal encoders and precomputed graph features. 

\subsection{Graph Neural Diffusion}
The message passing mechanism is a core technique in graph neural networks (GNNs), where information is iteratively aggregated to central nodes of the direct neighborhood. This process is associated with a physical process called heat diffusion \cite{bochner1949diffusion}. New models have been established on this continuous formulation to address some limitations of classic GNNs. For example, GRAND \cite{chamberlain2021grand} generalized the graph attention based on an anisotropic diffusion. GRAND++ \cite{thorpe2022grand++} further extended the model with an additional source term.
DIFFormer \cite{wu2023difformer} developed a Transformer-like diffusion scheme with a global attention model. These works studied static graph-based tasks such as node classification, which differ from the spatiotemporal prediction problem in this paper.

{
\subsection{Spatiotemporal Transformers}
Spatiotemporal Transformers have emerged as a powerful framework for modeling complex spatial-temporal dependencies inherent in urban systems. Their self-attention mechanisms enable the capture of long-range temporal patterns and dynamic spatial relationships, which are critical for urban computing applications \cite{wen2022transformers}. 
For example, spatiotemporal Transformer networks (STTNs) \cite{xu2020spatial} capture dynamic spatial dependencies and long-range temporal patterns, leading to enhanced accuracy in short-term traffic prediction. 
Additionally, lightweight architectures such as ST-TIS \cite{li2022lightweight} reduce computational cost while maintaining high accuracy through information fusion and region sampling. 
Beyond traffic data, they have been utilized in broader urban computing contexts, such as urban mobility modeling \cite{wang2024cola,yuan2024foundation}, urban visual scene understanding \cite{zhou2022mtanet,zhou2023embedded,zhou2023mmsmcnet,zhou2024mdnet,ma2025novel}, as well as environmental and energy monitoring \cite{ma2023stnet,dai2025citytft}. 
These developments underscore the versatility and effectiveness of spatiotemporal Transformers in addressing diverse challenges within urban networks.
}

{
\subsection{Summary of Challenges}
Existing spatiotemporal prediction studies fall into three main streams. Data-driven models based on STGNNs and Transformers deliver strong predictive performance but rely on deep stacked “black‑box” modules that obscure physical interpretability and incur high computational cost. Scalable graph‑based approaches such as random-walk diffusion with echo state networks and preprocessing‑based sensor embeddings can enhance efficiency on medium-scale networks but depend on complex temporal encoders, precomputed features, and sampling strategies that limit generalizability. Continuous diffusion‑inspired architectures like GRAND and DIFFormer establish principled links to physical processes but focus mainly on static or small graphs, hindering their applicability to large urban systems. Collectively, these lines of work expose a persistent trade‑off between effectiveness and efficiency, driven by the absence of a unified, physics‑informed framework. Motivated by this gap, our study formulates an energy‑regularized neural diffusion perspective, which bridges micro- and macro-scale dynamics to yield an interpretable, scalable Transformer‑like model for large‑scale prediction.
}

\begin{figure*}[!htbp]
  \centering
  \includegraphics[width=0.99\textwidth]{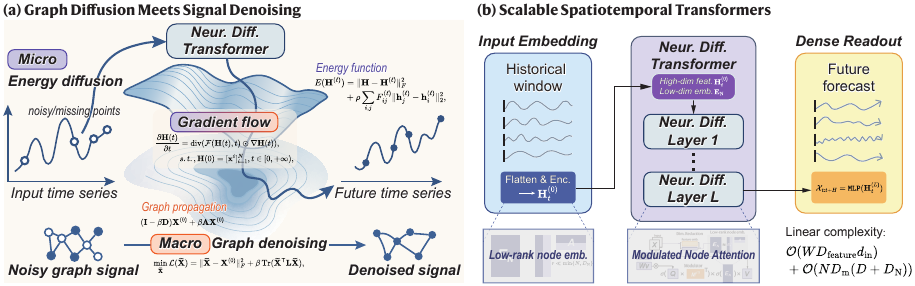}
  \caption{{Overview of the proposed theoretical framework and the model architecture. (a) Our theoretical analysis links the energy diffusion scheme with the graph signal denoising process. (b) This analysis inspires the design of a scalable spatiotemporal Transformer model with linear complexity.}}
  \label{fig:overview}
\end{figure*}

\section{Motivation and Theoretical Framework}\label{sec:motivation}
Before introducing the model, we articulate our motivation by establishing a theoretical framework and observing evidence from data. These insights shed light on model design and guide us to propose a novel class of architecture to achieve both efficiency and effectiveness for large urban networks.

\subsection{Notation and Preliminary}
Consider a \textit{spatiotemporal graph} (STG) with fixed topology, the \textit{node} set $\mathcal{V}$ represents the union of all sensors with $|\mathcal{V}|=N$, and the adjacency matrix $\mathbf{A}\in\mathbb{R}^{N\times N}$ {with its entry being $a_{i,j}$} prescribes the connectivity between nodes. Each node observes a time-varying graph signal and we denote the observation at time step $t$ of node $i$ as $\mathbf{x}_t^i\in\mathbb{R}^{d_{\text{in}}}$. 
Without ambiguity, we also denote $x_i(t)$ the scaler nodal state of the node $i$ and time $t$, {i.e., $\text{If }d_{\rm in}=1,\;x_i(t)\equiv x_t^i,\text{otherwise use }\mathbf{x}_t^i\in\mathbb{R}^{d_{\rm in}}.$}
Then we use the matrix $\mathbf{X}_t\in\mathbb{R}^{N\times d_{\text{in}}}$ to indicate the graph state at a single step and the tensor $\mathcal{X}_{t:t+T}\in\mathbb{R}^{N\times T\times d_{\text{in}}}$ to represent all the signals within the interval $\{t,t+1,\dots,t+T\}$. Given the STG within a historical window $W$ as $\mathcal{G}_{t-W:t}=(\mathcal{X}_{t-W:t};\mathbf{A})$, the prediction problem can be formulated as learning a multivariate forecaster $\mathcal{F}$:
\begin{equation}
{
    \widehat{\mathcal{X}}_{t+1:t+H}=\mathcal{F}(\mathcal{G}_{t-W:t}),
    }
\end{equation}
where $H$ is the horizon and {$\widehat{\mathcal{X}}_{t+1:t+H}\in\mathbb{R}^{N\times H\times d_{\text{out}}}$} is the prediction. 
The forecaster $\mathcal{F}$ is learned through minimization of the supervision loss over all nodes on the graph:
\begin{equation}
    {
    \mathcal{L}_{t+1:t+H}=\sum_{h=t+1}^{t+H}\sum_{i=1}^N\Vert\widehat{\mathbf{x}}_h^i-\mathbf{x}_h^i \Vert_1.
    }
\end{equation}

\subsection{A General Class of Network Dynamical Model}

An urban sensor network can be characterized as a complex networked system consisting of two interdependent parts: the network topology and the network dynamics. The former includes links and interconnected nodes, and the latter is specified by some governing equations \cite{prasse2022predicting}.
Dynamics on networks describe a wide range of urban phenomena, such as the propagation of traffic congestion, interactions on electrical grids, and activities on production networks. To model these networks, we consider a general class of dynamics model \cite{prasse2022predicting}:
\begin{equation}\label{eq:dynamics}
{
    \frac{\mathrm{d}x_i(t)}{\mathrm{d}t}=\mathcal{E}_i(x_i(t))+\sum_{j=1}^Na_{i,j}\mathcal{I}(x_i(t),x_j(t)),\forall i=\{1,\dots,N\},
    }
\end{equation}
where $\mathcal{E}_i(\cdot)$ prescribes the self-dynamics of node $i$ and can be heterogeneous across the network, and $\mathcal{I}(\cdot)$ is the function that describes the interaction between node $i$ and its neighbors.

In urban networks, $N$ can be very large, making precise identification and prediction of dynamics difficult.
Fortunately, recent studies reveal that the dynamics of large networks reside in a low-dimensional subspace \cite{thibeault2024low,wu2024predicting}. Some dimensional reduction techniques can be used to approximate the full dynamics using reduced-order structures, e.g., the proper orthogonal decomposition (POD) {for scalar function}:
\begin{equation}\label{eq:POD}
    \mathbf{x}(t)\approx\sum_{k=1}^r\alpha_k(t)\phi_k,
\end{equation}
where $\mathbf{x}(t)=[x_i(t)]_{i=1}^N$ is the stack of nodal states, $\phi_1,\dots,\phi_r$ are orthonormal vectors, and $\alpha_k(t)$ is a time-dependent factor. We will detail the method to obtain the POD in Section \ref{sec:low-rank-emb}.

While observations are from the real world, authentic functions $\mathcal{E}_i(\cdot)$ and $\mathcal{I}(\cdot)$ are usually unknown. Additionally, observed graph structures can be incomplete or noisy due to error-prone data collection.
Thus, we adopt surrogate neural models for Eq. \eqref{eq:dynamics} in latent spaces. Graph neural networks (GNNs) and Transformers are widely adopted in this context. However, they face limitations in interpretability and efficiency. To inspire a new class of architecture for addressing these issues, we theoretically analyze the neural message passing mechanism and endeavor to unlock the ``black box''.

\subsection{Graph Neural Diffusion with Energy Regularization}\label{sec:energy_diffusion}
Generally, we have found two principled ways to analyze the network dynamics with interpretability. First, we study the \textbf{ microscopic} behavior of the system.
We start by revisiting the diffusion equations from first principles. The diffusion equations describe how certain quantities of interest, such as mass or heat, disperse spatially as a function of time, according to the law of Fick and the law of mass conservation \cite{bochner1949diffusion}.
There is a natural analogy between the message passing of node in the graph and the (heat) diffusion on the Riemannian manifold.

Formally, the quantity spreads out from the locations with high concentrations to others with mass continuity. Given node representations processed by neural models as the physical quantity and update of node representations $\mathbf{h}_i(t)\in\mathbb{R}^d$ per layer as flux through time, the diffusion process is described by a partial differential equation with initial conditions \cite{rosenberg1997laplacian}:
\begin{equation}\label{eq:diffusion}
\begin{aligned}
    \frac{\partial \mathbf{H}(t)}{\partial t}&=\operatorname{div}(\mathcal{F}(\mathbf{H}(t),t)\odot\nabla\mathbf{H}(t)), \\
    &s.t., \mathbf{H}(0)=[\mathbf{x}^i]_{i=1}^N,t\in[0,+\infty),
\end{aligned}
\end{equation}
where $\mathbf{H}(t)=[\mathbf{h}_i(t)]_{i=1}^N$, $\operatorname{div}$ is the divergence operator that computes the total mass changes at certain locations, $\nabla$ is the gradient operator measures the difference over space, and $\mathcal{F}(\mathbf{H}(t),t):\mathbb{R}^{N\times d}\times[0,+\infty)\mapsto\mathbb{R}^{N\times N}$ denotes the \textit{diffusivity} function that determines the diffusion intensity between nodes at time $t$. As a discrete realization, Eq. \eqref{eq:diffusion} can be written as an explicit form using the differential operators on graphs:
\begin{equation}\label{eq:graph-diffusion}
    \frac{\partial \mathbf{h}_i(t)}{\partial t}=\sum_{j\in\mathcal{N}(i)}F_{ij}(\mathbf{H}(t),t)(\mathbf{h}_j(t)-\mathbf{h}_i(t)), 
\end{equation}
where $\{F_{ij}\}_{i,j}$ is the diffusivity matrix associated with $\mathcal{F}$ and $\mathcal{N}(i)$ is connected neighbors of node $i$. Eq. \eqref{eq:graph-diffusion} characterizes the graph neural diffusion of instance evolution in continuous dynamics. It can be solved using numerical methods, such as the explicit Euler scheme with difference step size $\delta$:
\begin{equation}\label{eq:euler-scheme}
{
\begin{aligned}
    \mathbf{h}_i^{(\ell+1)}&=\mathbf{h}_i^{(\ell)}+{\delta}\sum_{j\in\mathcal{N}(i)}F_{ij}^{(\ell)}(\mathbf{h}_j^{(\ell)}-\mathbf{h}_i^{(\ell)}), \\
    &=\mathbf{h}_i^{(\ell)}-{\delta}\left(\operatorname{diag}(\sum_j F_{ij}^{(\ell)})-\mathbf{F}^{(\ell)}\right)\mathbf{h}_i^{(\ell)},\\
    &=(\mathbf{I}-{\delta}\operatorname{diag}(\sum_j F_{ij}^{(\ell)}))\mathbf{h}_i^{(\ell)}+\delta\mathbf{F}\mathbf{h}_i^{(\ell)},
\end{aligned}
}
\end{equation}
which constructs the layer-wise updating rule of the graph neural diffusion model \cite{chamberlain2021grand}. {The first term is a self-updating source with a residual connection with the last state, and the second term aggregates the information from all neighborhoods on the graph.}
Eq. \eqref{eq:euler-scheme} is a discrete neural network model of the network dynamics in Eq. \eqref{eq:dynamics}.
As a generalized case, we consider the underlying graph is densely connected, i.e., $\mathcal{N}(i)=\mathcal{V}$, and the diffusivity is a latent variable condition on the layerwise nodal representation to be inferred.
In particular, the microbehavior in this physical system is controlled by a global energy that imposes some constraints on the direction of the evolution towards an equilibrium state \cite{di2022understanding,wu2023difformer}.
The regularized Dirichlet energy is used to quantify the total variability of quantities in the {graph-structured} system:
\begin{equation}\label{eq:energy}
{
    E(\mathbf{H})=\Vert\mathbf{H}-\mathbf{H}^{(\ell)}\Vert_F^2+\rho\sum_{i,j}F_{ij}^{(\ell)}\Vert\mathbf{h}_j-\mathbf{h}_i\Vert_2^2,
    }
\end{equation}
where $\mathbf{H}^{(\ell)}=[\mathbf{h}_i^{(\ell)}]_{i=1}^N$.
The first term regularizes the consistency between layerwise embedding {$\mathbf{H}^{(\ell)}$ and system variable $\mathbf{H}$ before propagation}, and the second term controls the global smoothness (total variation) of {node states (variables)} on the graph.
The following proposition shows how the energy controls the microbehavior of per layer updates.
\begin{proposition}
Gradient flows to reduce the energy defined in Eq. \eqref{eq:energy} lead to the iterative diffusion propagation scheme. 
\end{proposition}
\begin{proof}
    Considering that the physical system tends to converge at a steady point with energy minimized, we study the gradient flow to minimize {$E(\mathbf{H})$} for all nodes:
    \begin{equation}\label{eq:energy_min}
    {
    \begin{aligned}
        &\mathbf{H}^{(\ell+1)}=\mathbf{H}^{(\ell)}-\alpha\frac{\partial E(\mathbf{H})}{\partial \mathbf{H}}\vert_{\mathbf{H}=\mathbf{H}^{(\ell)}}, \\
        &=\mathbf{H}^{(\ell)}-2\alpha(\mathbf{H}-\mathbf{H}^{(\ell)}) -\alpha\rho\frac{\partial \sum_{i,j}F_{ij}^{(\ell)}\Vert\mathbf{h}_j-\mathbf{h}_i\Vert_2^2}{\partial \mathbf{H}}\vert_{\mathbf{H}=\mathbf{H}^{(\ell)}}, \\
        &=\mathbf{H}^{(\ell)}-2\alpha\rho(\mathbf{D}^{(\ell)}-\mathbf{F}^{(\ell)})\mathbf{H}^{(\ell)}, \\
        &=(\mathbf{I}-\delta'\mathbf{D}^{(\ell)})\mathbf{H}^{(\ell)}+\delta'\mathbf{F}^{(\ell)}\mathbf{H}^{(\ell)},\\
        \end{aligned}
    }
    \end{equation}
    where $\delta'=2\alpha\rho$ and {$\mathbf{D}^{(\ell)}=\operatorname{diag}(\sum_j F_{ij}^{(\ell)})$}. Note that this is the matrix version of Eq. \eqref{eq:euler-scheme} and thus the relationship holds.
\end{proof}

The above analysis indicates that we can use neural networks to model the network dynamics in the latent space and enforce constraints by designing appropriate energy measures. However, in practice, energy is not intuitive for urban networks. In the following, we reveal that the energy reduction can be understood as a solution to a graph denoising problem.

\subsection{Spatiotemporal Graph Signal Denoising}\label{sec:graph-denoising}

From another perspective, we investigate the \textbf{macroscopic} phenomenon of the above process. Let us consider estimating a low-frequency component $\widetilde{\mathbf{X}}$ from the observed noised data $\mathbf{X}^{(0)}$, the graph regularized least square problem is given by:
\begin{equation}\label{eq:graph-RLS}
    \min_{\widetilde{\mathbf{X}}} \mathcal{L}(\widetilde{\mathbf{X}})=\Vert \widetilde{\mathbf{X}}-\mathbf{X}^{(0)}\Vert_F^2 + \beta \operatorname{Tr}(\widetilde{\mathbf{X}}^{\mathsf{T}}\mathbf{L}\widetilde{\mathbf{X}}),
\end{equation}
which formulates the graph signal denoising problem \cite{nt2019revisiting} defined over $N$ nodes, and $\mathbf{L}$ is the {unnormalized} Laplacian. The first penalty guides $\widetilde{\mathbf{X}}$ to be close to $\mathbf{X}^{(0)}$ and the second term is the Laplacian regularization that encourages a smooth signal.
To recover the smooth signal $\widetilde{\mathbf{X}}$, we derive the optimal solution to Eq. \eqref{eq:graph-RLS} and develop the following proposition.

\begin{proposition}
The final state of the diffusion process in Eq. \eqref{eq:energy_min} is a denoised smooth graph signal.
\end{proposition}

\begin{proof}
Let $\frac{\partial \mathcal{L}(\widetilde{\mathbf{X}})}{\partial \widetilde{\mathbf{X}}}=0$, we have:
    \begin{equation}\label{eq:GRLS_solution}
    \begin{aligned}
        &2\beta\mathbf{L}\widetilde{\mathbf{X}}+2(\widetilde{\mathbf{X}}-\mathbf{X}^{(0)})=0, \\
        &\Rightarrow \widetilde{\mathbf{X}}=(\mathbf{I}+\beta\mathbf{L})^{-1}\mathbf{X}^{(0)}.
    \end{aligned}
    \end{equation}
Obtaining the inverse matrix of a large graph can incur high complexity. Therefore, we consider using the {first‐order} Taylor expansion to approximate Eq. \eqref{eq:GRLS_solution}. {Given a small enough $\beta$}:
\begin{equation}
    \begin{aligned}
        \widetilde{\mathbf{X}}&=(\mathbf{I}+\beta\mathbf{L})^{-1}\mathbf{X}^{(0)}\approx(\mathbf{I}-\beta\mathbf{L})\mathbf{X}^{(0)}, \\
        &= (\mathbf{I}-\beta(\mathbf{D}-\mathbf{A}))\mathbf{X}^{(0)}= (\mathbf{I}-\beta\mathbf{D})\mathbf{X}^{(0)}+\beta\mathbf{A}\mathbf{X}^{(0)}.
    \end{aligned}
\end{equation}
If we process the static adjacency matrix as the composition of layer-specific diffusivity matrices $\mathbf{A}=\mathbf{F}^{(0)}\circ\cdots\circ\mathbf{F}^{(L)}$, 
this is the result of layer-wise message propagation with residual connection of the initial nodal state described in Eq. \eqref{eq:energy_min}.
\end{proof}

Putting the two schools of viewpoint together, we can have a unified view of the physical process used to predict the network dynamics. We develop the following corollary.

\begin{corollary}
    The message-passing based network dynamics prediction model follows a per-layer energy diffusion scheme to iteratively denoise the observed time series, achieving prediction using the recovered signal in latent spaces.
\end{corollary}

This physical prior can guide the design of model architecture. The key is to design a proper diffusivity description that is applicable to large-scale urban networks. There are several natural choice for $\mathbf{F}(\mathbf{H}(t),t)$: (1) If $\mathbf{F}(\mathbf{H}(t),t)=\mathbf{I}$, Eq. \eqref{eq:euler-scheme} reduces to a MLP model with residual connection; (2) If $\mathbf{F}(\mathbf{H}(t),t)=\mathbf{W}$ as a learnable matrix, it becomes a MLP-Mixer model for graphs \cite{he2023generalization}; (3) If $\mathbf{F}(\mathbf{H}(t),t)$ is specified as the observed graph $\mathbf{A}$, it results in a standard GNN; (4) If we allow $\mathbf{F}(\mathbf{H}(t),t)$ as a layer-dependent latent variable and infer it using the node representation $\mathbf{H}^{(\ell)}$, then it generates a (graph) Transformer model.
To balance both expressivity and efficiency for large urban networks, in Section \ref{sec:model} we design a new class of neural architecture based on this inductive bias.

\subsection{Empirical Observations}
\begin{figure}[!htbp]
  \centering
  \includegraphics[width=0.99\columnwidth]{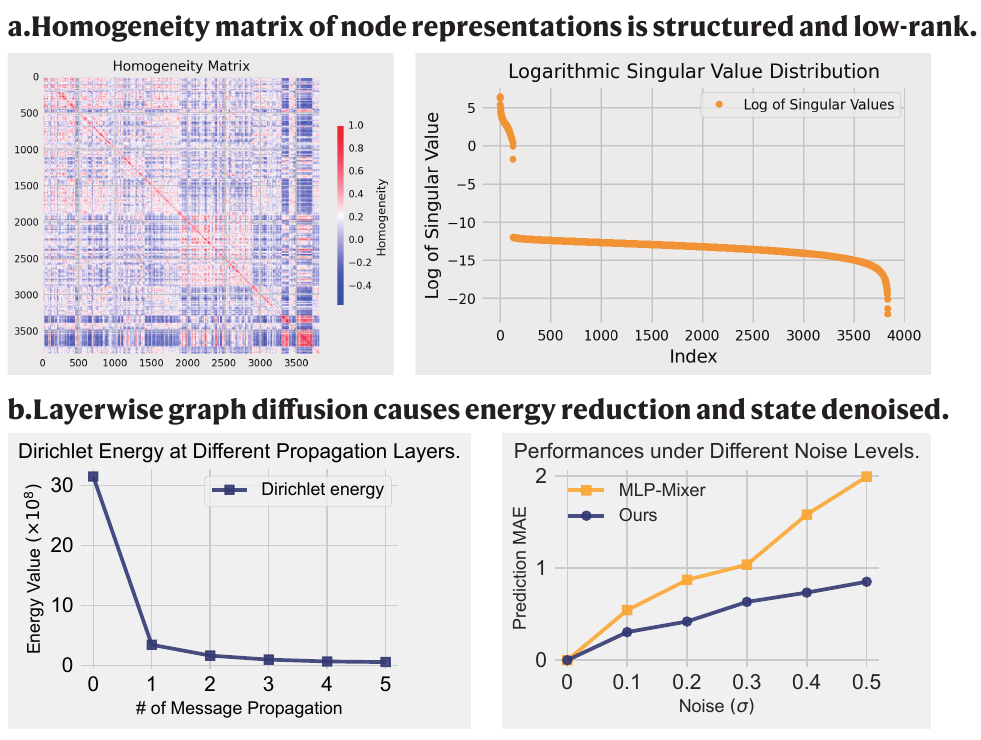}
  \caption{Empirical observations using real-world and synthetic data.}
  \label{fig:motivation}
\end{figure}

To provide justification for the above analysis, we use empirical data examples to show evidence for the graph denoising process and low-dimensional structures in urban networks. First, we study the large-scale traffic network from California (more detailed data descriptions are given in Section \ref{sec:exp}). In Fig. \ref{fig:motivation} (a), we obtain the learned node embedding vectors (see Section \ref{sec:low-rank-emb}) and compute the cosine similarity as the homogeneity score for node representations. This matrix indicates the collective patterns on the graph (network) and shows significant structures. More intuitively, we display its singular values and find a clear truncated distribution, i.e., a low-rank pattern. This means that the dynamics of node representations is controlled by low-dimensional manifolds.

Second, we study a microsystem described in \cite{cini2024taming}, which consists of a locality-aware graph polynomial vector autoregressive model to approximate the behavior of STGNNs:

\begin{equation}\label{eq:gpvar}
    \mathbf{H}_t=\sum_{l=0}^L\sum_{p=1}^P\Psi_{p,l}\mathbf{S}^{l}[\mathbf{X}_{t-p}\|\mathbf{u}_{t-p}], ~\mathbf{X}_t=\mathbf{e}\odot\xi(\mathbf{H}_t) + \eta_t,
\end{equation}
where $\Psi\in\mathbb{R}^{P\times L}$ is the collection of model parameters, $P$ is the total number of time lags, $L$ is the total order of graph propagation, $\eta_t\sim \mathcal{N}(0,\sigma^2\mathbb{I})$ is the Gaussian noise, $\xi$ is the nonlinear function, $\mathbf{H}_t$ is the hidden state at step $t$, $\mathbf{e}\in\mathbb{R}^{N_o}$ simulates the region-specific patterns, $\mathbf{S}^l$ is a graph Laplacian.

We adopt this prototype to evaluate the denoising effect and energy propagation. In Fig. \ref{fig:motivation} (b), we first calculate the Dirichlet energy at different propagation layers. As indicated by the energy trajectory, it decreases rapidly with increasing graph propagation. This echos the layerwise diffusive effects in Section \ref{sec:energy_diffusion}.
Then, we compare the denoising effect of our model with a MLP-Mixer model \cite{chen2023tsmixer}.
The noise level is measured by the standard deviation of the white noise added to the features. By gradually increasing the noise level, both models show higher errors. However, our model is more robust to noise, showing an effective denoising effect.

\section{Proposed Model}\label{sec:model}
{
The analysis in Section \ref{sec:motivation} shows that (1) the evolution of spatiotemporal graph signals can be viewed as an energy-driven diffusion process that iteratively denoises observations towards a smooth equilibrium (Eqs. \eqref{eq:euler-scheme}-\eqref{eq:GRLS_solution}), and (2) the underlying dynamics lie in a low-dimensional manifold amenable to POD (Eq. \eqref{eq:POD}).
Therefore, we directly instantiate these physical principles. 
Concretely, we first compress the raw time series of each node into a reduced order embedding with a POD‑inspired node adapter (Sections \ref{sec:low-rank-emb} and \ref{sec:input_enc}), ensuring that the model captures the intrinsic low‑rank structure of the network. 
Next, we realize the explicit Euler discretization of the diffusion PDE (Eq. \eqref{eq:euler-scheme}) as a multilayer neural diffusion block, enforcing iterative energy minimization to progressively denoise and propagate node representations. 
To maintain scalability on large graphs, we approximate the resulting diffusivity-driven attention kernel with a low‑rank modulated node attention mechanism (Eq. \eqref{eq:low_rank_atten}), reducing complexity from $\mathcal{O}(N^2)$ to nearly linear in $\mathcal{O}(N)$. 
Finally, the denoised embeddings are decoded back into multistep predictions with a lightweight MLP. 
Together, these components define our scalable spatiotemporal Transformer (ScaleSTF) model to predict the dynamics of large-scale urban networks with both accuracy and efficiency.
}


\subsection{Overall Framework}
{As shown in Fig. \ref{fig:overview} (b), the overall process of ScaleSTF has three stages, which can be formulated as follows:}
\begin{equation}
\begin{aligned}
    &\mathbf{h}_t^{i,(0)}=\texttt{InputEmbedding}(\mathbf{X}_{t-W:t}^{i}),~\forall i=\{1,\dots,N\}, \\
    &\mathbf{H}_t^{(\ell)}=\texttt{NeuralDiff}(\mathbf{H}_t^{(\ell-1)}),~\forall \ell=\{1,\dots,L\}, \\
    &\mathcal{X}_{t:t+H}=\texttt{MLP}(\mathbf{H}_t^{(L)}).
\end{aligned}    
\end{equation}
where $\texttt{InputEmbedding}(\cdot):\mathbb{R}^{W\times d_{\text{in}}}\mapsto\mathbb{R}^{D}$ summarizes the time series as a dense latent vector, $\mathbf{H}_t^{(\ell)}=\{\mathbf{h}_t^{0,(\ell)},\dots,\mathbf{h}_t^{N,(\ell)}\}\in\mathbb{R}^{N\times D}$ is the set of node representations in the $\ell$-th layer, $\texttt{MLP}(\cdot):\mathbb{R}^D\mapsto\mathbb{R}^{H\times d_{\text{out}}}$ generates the multistep predictions and $\texttt{NeuralDiff}$ propagates the message of all node pairs.
Next, we will elaborate on the detailed design strategies on the structure of ScaleSTF.

\subsection{Observation Encoding}\label{sec:input_enc}
$\texttt{InputEmbedding}(\cdot)$ converts the time series of a sensor to node states in the latent space using a neural mapping layer and combines it with several learnable embeddings.  
However, expanding the input dimension $d_{\text{in}}$ to a large latent dimension can produce a large feature tensor $\mathcal{H}\in\mathbb{R}^{N\times W\times D}$ with potential redundancy and increased complexity. Instead, ScaleSTF compresses the feature embedding into a reduced-order vector and concatenates it with spatiotemporal embeddings:
\begin{equation}
\begin{aligned}
    &\mathbf{z}_t^{i,(0)}= \textbf{W}^{(0)}\texttt{Flatten}(\mathbf{X}_{t-W:t}^{i})+\textbf{b}^{(0)},~\forall i=\{1,\dots,N\}, \\
    &\mathbf{h}_t^{i,(0)} = \texttt{Concat}(\mathbf{z}_t^{i,(0)} ; \mathbf{e}^{\text{N}}_i ; \mathbf{e}^{\text{TiD}}_t ; \mathbf{e}^{\text{DiW}}_t), \\
\end{aligned}
\end{equation}
where $\texttt{Flatten}(\cdot):\mathbb{R}^{W\times d_{\text{in}}}\mapsto\mathbb{R}^{Wd_{\text{in}}}$ folds the last dimension of the tensor, $\mathbf{W}^{(0)}\in\mathbb{R}^{D_{\text{feature}}\times Wd_{\text{in}}}$, $\mathbf{b}^{(0)}\in\mathbb{R}^{D_{\text{feature}}}$ constitute the shared feature transformation, $\mathbf{e}^{\text{N}}_i\in\mathbb{R}^{D_{\text{N}}}$, $\mathbf{e}^{\text{TiD}}_t\in\mathbb{R}^{D_{\text{TiD}}}$, $\mathbf{e}^{\text{DiW}}_t\in\mathbb{R}^{D_{\text{DiW}}}$ are the learnable node, time-in-day, and day-in-week embeddings respectively \cite{shao2022spatial}, and $\mathbf{h}_t^{i,(0)}\in\mathbb{R}^{D}$ is the final node representation. 

To increase the model capacity and add nonlinearity, a two-layer MLP is applied to all node representations subsequently:
\begin{equation}
   \mathbf{H}_t^{(0)} = \sigma(\mathbf{W}^{(2)}(\sigma(\mathbf{W}^{(1)}\mathbf{H}_t^{(0)}+\mathbf{b}^{(1)}))+\mathbf{b}^{(2)})+\mathbf{H}_t^{(0)}.
\end{equation}

\subsection{Low-Rank Adapted Node Embedding}\label{sec:low-rank-emb}
On the one hand, the dimension $D_{\text{N}}$ of the node embedding should be large enough to ensure the distinguishability of the node representations. For example, trained index embeddings should guarantee $N!$ possible node permutations to ensure higher expressivity.
On the other hand, as indicated in Fig. \ref{fig:redundancy}, introducing a learnable vector with a large enough dimension to each node can significantly increase the rank of node representations, thereby causing overparameterization.

Recall that network dynamics are in a subspace smaller than the network dimension and these low-dimensional structures can be obtained through POD in Eq. \eqref{eq:POD}. This provides a feasible way to obtain these low-dimensional embeddings. Formally, scalar functions $\alpha_k(t)$ in Eq. \eqref{eq:POD} are obtained by projecting the nodal state $\mathbf{X}_{t}$ on the respective agitation mode:
\begin{equation}
    \alpha_k(t) = \phi_k^{\mathsf{T}}\mathbf{X}_{t},
\end{equation}
where the orthonormal modes are typically calculated using singular vectors of the nodal state matrix. However, computing the singular value decomposition has high computational complexity, especially for large matrices.

To address this issue, we suggest using learnable matrices to compose node-specific patterns in a low-dimensional space. 
Specifically, we assign a learnable \textit{node adapter} shared by all nodes as $\mathbf{P}\in\mathbb{R}^{r\times D_{\text{N}}}$, where $r\ll \min\{N, D_{\text{N}}\}$ is the rank, then the composed low-rank adapted embedding (LRAE) is given as follows to approximate the matrix version of Eq. \eqref{eq:POD}:
\begin{equation}\label{eq:lare}
    \mathbf{E}_{\text{N}}\approx \mathbf{E}_{\text{N}}^r \mathbf{E}_{\text{N}}^{r,\mathsf{T}}\mathbf{X}_{t} 
    =\mathbf{E}_{\text{N}}^r\mathbf{P}\in\mathbb{R}^{N\times D_{\text{N}}},
\end{equation}
where $\mathbf{E}_{\text{N}}^r\in\mathbb{R}^{N\times r}$ is the learnable dictionary of node-specific parameters. This low-rank reparameterization assumes that the learned node-specific patterns reside in a low intrinsic dimension. LRAE allows the shared model to capture individual patterns by optimizing rank-decomposed matrices rather than dense matrices, alleviating the difficulty of parameter learning.

\subsection{Scalable Modulated Node Attention}
\begin{figure}[!htbp]
  \centering
  \includegraphics[width=0.9\columnwidth]{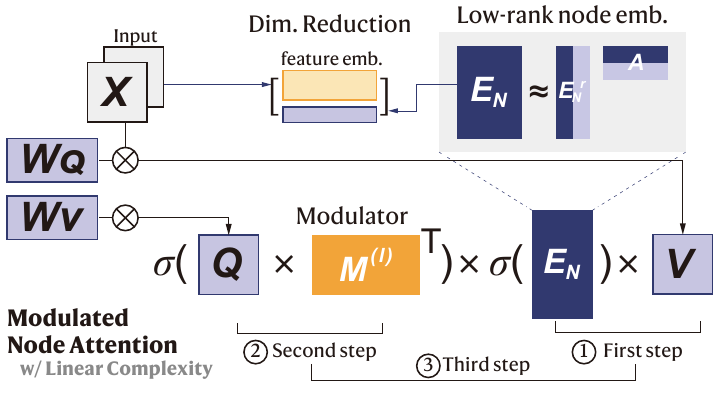}
  \caption{The computation flow of the proposed modulated node attention.}
  \label{fig:attention}
\end{figure}

According to the corollary in Section \ref{sec:graph-denoising}, \texttt{NeuralDiff} should base on a diffusivity measure to gradually denoise the graph and reduce the energy. To maximize the denoising effect on a global scale, we consider all pairwise diffusion. However, full attention computation of standard (graph) Transformers has $\mathcal{O}(N^2)$ space and time complexity, which is computationally prohibitive for large-scale graphs \cite{liu2024largest,nie2024channel}. 

To resolve the main computational bottleneck, we propose simplifying the self-attention with a lightweight node attention.
Recall that given the hidden representation at $(\ell-1)$-th layer, the canonical \texttt{Transformer} block is formulated as follows:
\begin{equation}
\begin{aligned}
    \mathbf{H}_t^{(\ell-1)}&=\texttt{LayerNorm}(\mathbf{H}_t^{(\ell-1)}+\mathbf{A}_s(\mathbf{H}_t^{(\ell-1)})\mathbf{H}_t^{(\ell-1)}\mathbf{W}_{\text{V}}), \\
     \mathbf{H}_t^{(\ell)}&=\texttt{LayerNorm}(\mathbf{H}_t^{(\ell-1)}+\texttt{MLP}(\mathbf{H}_t^{(\ell-1)})), \\
\end{aligned}
\end{equation}
with $\mathbf{W}_{\text{V}}\in\mathbb{R}^{D\times D_{\text{m}}}$ and $\mathbf{A}_s(\mathbf{H})\in\mathbb{R}^{N\times N}$ being the self-attention matrix defined as:
\begin{equation}\label{eq:full_att}
\begin{aligned}
\mathbf{A}_s(\mathbf{H}_t^{(\ell-1)})&=\texttt{SelfAtten}(\mathbf{H}_t^{(\ell-1)},\mathbf{H}_t^{(\ell-1)},\mathbf{H}_t^{(\ell-1)})),\\
&=\texttt{Softmax}\left(\frac{\mathbf{H}_t^{(\ell-1)}\mathbf{W}_{\text{Q}}\mathbf{W}_{\text{K}}^{\mathsf{T}}\mathbf{H}_t^{(\ell-1),\mathsf{T}}}{\sqrt{D_{\text{m}}}}\right),\\
\end{aligned}
\end{equation}
where $\mathbf{W}_{\text{Q}},\mathbf{W}_{\text{K}}\in\mathbb{R}^{D\times D_{\text{m}}}$, and $D_{\text{m}}$ is the model dimension.

Since the key (query) matrices can be treated as dynamic node representations of data flows, we can \textbf{approximate them using static node representations}, i.e., the LRAE $\mathbf{E}^{\text{N}}$. To achieve this, we elaborate a layer-wise attentive \textit{modulator} $\mathbf{M}^{(\ell)}\in\mathbb{R}^{D_{\text{N}}\times D_{\text{m}}}$ that 
is learned end-to-end from data to decompose the attention matrix and approximate Eq. \eqref{eq:full_att} as:
\begin{equation}\label{eq:factorized_att}
\mathbf{A}_s(\mathbf{H}_t^{(\ell-1)})\approx\texttt{Softmax}\left(\frac{\mathbf{H}_t^{(\ell-1)}\mathbf{W}_{\text{Q}}\mathbf{M}^{(\ell-1),\mathsf{T}}\mathbf{E}_{\text{N}}^{\mathsf{T}}}{\sqrt{D_{\text{m}}}}\right).
\end{equation}

In practice, since $D_{\text{N}}<D_{\text{m}}$, Eq. \eqref{eq:factorized_att} admits a low-rank factorized attention matrix, which preserves the most significant correlations for network dynamics prediction. By further decoupling the node modulation $\mathbf{E}_{\text{N}}\mathbf{M}\in\mathbb{R}^{N\times D_{\text{m}}}$, we can obtain a simplified updating rule as:
\begin{equation}\label{eq:low_rank_atten}
\begin{aligned}
    &\mathbf{A}_s(\mathbf{H}_t)\mathbf{H}_t\mathbf{W}_{\text{V}}=\texttt{Softmax}\left(\frac{\mathbf{H}_t\mathbf{W}_{\text{Q}}\mathbf{M}^{\mathsf{T}}\mathbf{E}_{\text{N}}^{\mathsf{T}}}{\sqrt{D_{\text{m}}}}\right)\mathbf{H}_t\mathbf{W}_{\text{V}}, \\
    &\approx\texttt{Softmax}\left(\frac{\mathbf{H}_t\mathbf{W}_{\text{Q}}\mathbf{M}^{\mathsf{T}}}{\sqrt{D_{\text{m}}}}\right)\left(\texttt{Softmax}(\mathbf{E}_{\text{N}}^{\mathsf{T}})\mathbf{H}_t\mathbf{W}_{\text{V}}\right),
\end{aligned}
\end{equation}
where the superscript is omitted to ease the notation, and $\texttt{Softmax}(\mathbf{E}_{\text{N}}^{\mathsf{T}})$ encourages the right stochasticity of attention maps. Eq. \eqref{eq:low_rank_atten} is the final modulated node attention to achieve an efficient surrogate of the standard spatial attention for large-scale graphs. The computation flow is shown in Fig. \ref{fig:attention}.

{
Note that using low-rank factorization to approximate the full attention matrix is guaranteed with a bounded error. We provide the following analysis to elaborate on this property.
\begin{lemma}[The low-rankness of modulated attention matrix]
Given any $\mathbf{Q},\mathbf{K},\mathbf{V}\in\mathbb{R}^{N\times D}$ and $\mathbf{W}_{\text{Q}},\mathbf{W}_{\text{K}},\mathbf{W}_{\text{V}}\in\mathbb{R}^{D\times D_{\text{m}}}$, for any column vector $\mathbf{h}_{\text{V}}\in\mathbb{R}^{N}$ of $\mathbf{V}\mathbf{W}_{\text{V}}$, there exists a low-rank matrix $\Tilde{\mathbf{A}}\in\mathbb{R}^{N\times N}$ that satisfies:
\begin{equation}
    P(\|\Tilde{\mathbf{A}}\mathbf{h}_{\text{V}}^{\mathsf{T}}-\mathbf{A}\mathbf{h}_{\text{V}}^{\mathsf{T}} \| <\epsilon \|\mathbf{A}\mathbf{h}_{\text{V}}^{\mathsf{T}} \|)>1-\mathcal{O}(1),
\end{equation}
where the low-rank matrix can become $\Tilde{\mathbf{A}}=\sigma(\mathbf{Q}\mathbf{M}^{\mathsf{T}})\sigma(\mathbf{E}^{\mathsf{T}})$ with $\text{rank}(\mathbf{E})=\Theta(\text{log}(N))$.
\end{lemma}

\begin{proof}
We assume that the modulation matrix can be decomposed as $\mathbf{M}=\mathbf{E}^{\mathsf{T}}\mathbf{K}\in\mathbb{R}^{D_{\text{N}}\times D_{\text{m}}}$, then the simplified attention matrix in Eq. \eqref{eq:low_rank_atten} can be approximated as:
\begin{equation}
    \begin{aligned}
    \Tilde{\mathbf{A}}&=\sigma(\mathbf{Q}\mathbf{M}^{\mathsf{T}})\sigma(\mathbf{E}^{\mathsf{T}}), \\
    &=\sigma(\mathbf{Q}\mathbf{K}^{\mathsf{T}}\mathbf{E})\sigma(\mathbf{E}^{\mathsf{T}}),\\
    &\approx\sigma(\mathbf{Q}\mathbf{K}^{\mathsf{T}})\sigma(\mathbf{E}\mathbf{E}^{\mathsf{T}}).
    \end{aligned}
\end{equation}

If $\mathbf{E}\in\mathbb{R}^{N\times r}$ is a random projection matrix with i.i.d. entries sampled from a Gaussian $\mathcal{N}(0,1/r)$, it invokes the Johnson-Lindenstrauss condition in \cite{wang2020linformer} when $r=c\text{log}(N/\epsilon^2-\epsilon^3)$:
    \begin{equation*}
    \begin{aligned}
    P&(\|\sigma(\mathbf{Q}\mathbf{K}^{\mathsf{T}})\mathbf{E}\mathbf{E}^{\mathsf{T}}\mathbf{V}-\sigma(\mathbf{Q}\mathbf{K}^{\mathsf{T}})\mathbf{V} \| <\epsilon \|\sigma(\mathbf{Q}\mathbf{K}^{\mathsf{T}})\mathbf{V} \|) \\
    &>1-\mathcal{O}(1),
    \end{aligned}
    \end{equation*}
where $c$ is a constant. For detailed derivations, refer to \cite{wang2020linformer}.
\end{proof}
}

\subsection{Model Complexity}
ScaleSTF has \textbf{linear complexity} in both temporal and spatial dimensions. For temporal processing, ScaleSTF adopts $\texttt{MLP}$ to transform time series, which entails $\mathcal{O}(WD_{\text{feature}}d_{\text{in}})$ complexity.
For spatial processing, we can first calculate and store the right part of Eq. \eqref{eq:low_rank_atten}, then multiply it by the left part. Specifically, we can: (1) compute $\texttt{Softmax}(\mathbf{E}_{\text{N}}^{\mathsf{T}})\mathbf{H}_t\mathbf{W}_{\text{V}}$ in $\mathcal{O}(ND_{\text{N}}D)$ complexity; (2) compute $\texttt{Softmax}({\mathbf{H}_t\mathbf{W}_{\text{Q}}\mathbf{M}^{\mathsf{T}}})$ in $\mathcal{O}(ND_{\text{m}}(D+D_{\text{N}}))$ complexity; and then (3) multiply the two results in $\mathcal{O}(ND_{\text{m}}D_{\text{N}})$ complexity.
In general, ScaleSTF scales linearly with respect to spatial and temporal dimensions, making it efficient for large-scale networks.

\section{Experiments}\label{sec:exp}
This section performs evaluations using both real-world and synthetic large-scale networks. We compare ScaleSTF to advanced baselines in benchmark tasks covering networked urban systems from transportation, power production, to smart meters. Then discussions and analysis are provided.

\subsection{Experimental Setup}

\begin{table}[!htbp]
\centering
\centering
\setlength{\abovecaptionskip}{0.cm}
\setlength{\belowcaptionskip}{-0.0cm}
\caption{{Statistics of large-scale and medium-scale graph datasets.}}
\label{tab_datasets}
\small
\centering
\begin{small}
\setlength{\tabcolsep}{1pt}
\resizebox{0.99\columnwidth}{!}{
\begin{tabular}{c| c| c| c| c| c}
\toprule
 \multicolumn{1}{c|}{\textbf{Datasets}} & \textbf{Type} & \textbf{Steps} & \textbf{Nodes} & \textbf{Edges} & \textbf{Interval}\\
 \midrule
{GLA} & traffic volume & 525,888 & 3,834 & 98,703 & 15 min \\
{GBA} & traffic volume & 525,888 & 2,352 & 61,246 & 15 min \\
{PV-US} &  solar power & 52,560 & 5,016 & 417,199 & 30 min \\
{CER-En} &  smart meters & 52,560 & 6,435 & 639,369 & 30 min \\
\midrule 
{AirQuality} &  {PM2.5 pollutant} & {8,760} & {437} & {2,699} & {60 min} \\
{Elergone} & {load profiles} & {140,256} & {370} & {-} & {15 min} \\
\midrule 
{GP-VAR} &  synthetic & 30,000 & \multicolumn{2}{c|}{tunable}  & N.A. \\
\bottomrule
\end{tabular}}
\end{small}
\end{table}

\noindent\textbf{Datasets.}
We conduct evaluations on four large-scale networked urban datasets in the real world, including GLA and GBA from the LargeST traffic flow benchmark \cite{liu2024largest}, PV-US from the National Renewable Energy Lab, and CER-En from the Irish Commission for Energy Regulation Smart Metering Project.
{Two medium-scale datasets including AirQuality \cite{zheng2013u} and Elergone are used to benchmark our model with the state-of-the-art.}
A synthetic graph system GP-VAR \cite{cini2024taming} is also adopted to control the experimental conditions. 
Brief descriptions about the adopted datasets are shown in Tab. \ref{tab_datasets}.

\begin{table*}
\renewcommand{\arraystretch}{1.1} 
    \centering
    \setlength{\abovecaptionskip}{0.cm}
    \setlength{\belowcaptionskip}{-0.0cm}
    \setlength{\tabcolsep}{5.pt}
    \caption{Results of large-scale spatiotemporal graph forecasting on GLA, GBA, PV-US, and CER-En benchmarks.}
    \label{tab:main}
\resizebox{0.99\textwidth}{!}{
    \begin{tabular}{c|c|ccc>{\columncolor{gray25}}c|ccc>{\columncolor{gray25}}c|ccc>{\columncolor{gray25}}c|ccc>{\columncolor{gray25}}c}
    \toprule 
    \multicolumn{2}{c|}{\textbf{Dataset}} & \multicolumn{4}{c|}{\textbf{GLA}} & \multicolumn{4}{c|}{\textbf{GBA}} & \multicolumn{4}{c|}{\textbf{PV-US}}& \multicolumn{4}{c}{\textbf{CER-En}} \tabularnewline
    \midrule 
    \textbf{Method} & \textbf{Metric} & \textbf{@3} & \textbf{@6} & \textbf{@12} & \textbf{Avg.} & \textbf{@3} & \textbf{@6} & \textbf{@12} & \textbf{Avg.} & \textbf{@3} & \textbf{@6} & \textbf{@12} & \textbf{Avg.} & \textbf{@3} & \textbf{@6} & \textbf{@12} & \textbf{Avg.} \tabularnewline
    \midrule 
\multirow{2}{*}{DCRNN} & MAE & 18.33 & 22.70 & 29.45  & 22.73 & 18.25 & 22.25 & 28.68 & 22.35 & 2.42 &3.70  & 5.73 & 3.76 & 0.27 & 0.29 & 0.32 & 0.29  \tabularnewline
 & RMSE  & 29.13 & \underline{35.55} & 45.88 & 35.65 & 29.73 & 35.04 & 44.39 & 35.26 & 6.17 & 9.50 & 14.29 & 10.13 & 0.68 & 0.74 & 0.80 & 0.72 \tabularnewline
\midrule 
\multirow{2}{*}{AGCRN} & MAE & 17.57 & \underline{20.79} & \underline{25.01} & \underline{20.61}  & 18.11 & \underline{20.86} & \underline{24.06} & \underline{20.55} &  2.59 & 3.27 & 4.00 &3.15 & 0.28 & 0.29 &  0.31 &  0.29 \tabularnewline
 & RMSE  & 30.83 & 36.09 & 44.82 & 36.23 & 30.19 & 34.42 & \underline{39.47} & 33.91 & 6.21 & 7.73 & 9.18 & 7.53 & 0.65 & \underline{0.68} & \underline{0.74} & 0.68  \tabularnewline
\midrule 
\multirow{2}{*}{STGCN} & MAE & 19.87 & 22.54 & 26.48 & 22.48 & 20.62 & 23.19 & 26.53 & 23.03 & 2.61 &4.00  &5.92  & 3.95 & 0.27 & 0.30 & 0.33 &  0.30  \tabularnewline
 & RMSE  & 34.01 & 38.57 & 45.61 & 38.55 & 33.81 & 37.96 & 43.88 & 37.82 &6.58  & 10.17 & 14.63 & 10.56 & 0.69 & 0.75 & 0.82 & 0.74  \tabularnewline
\midrule 
\multirow{2}{*}{GWNet} & MAE & \underline{17.30} & 21.22 & 27.25 & 21.23 & \underline{17.74} & 20.98 & 25.39 & 20.78 & \underline{2.05} & 3.02 & 3.82 &  2.87 & 0.27 & 0.29 & 0.32 & 0.29   \tabularnewline
 & RMSE  & \underline{27.72} & 33.64 & \underline{43.03} & \underline{33.68} & \underline{28.70} & \underline{33.50} & 40.30 & \underline{33.32} & \underline{5.64} & 7.80 & 9.53 & 7.63 & 0.68 & 0.74  & 0.80 & 0.72  \tabularnewline
\midrule 
\multirow{2}{*}{TSMixer} & MAE & 21.76 & 27.06 & 31.59 &25.86 & 18.95 & 22.27 & 25.34 & 21.63 & 2.11 & \underline{2.86} & 3.72 & \underline{2.80} & \underline{0.26} & \underline{0.28} & \underline{0.30} & \underline{0.28}  \tabularnewline
 & RMSE  & 33.72 & 40.76 &47.40  & 39.94 & 30.46 & 35.65 & 40.11  & 34.90 & 5.89 & \underline{7.53} & 8.95 & \underline{7.27} & \underline{0.63} & \underline{0.68} & \underline{0.74} & \underline{0.66}  \tabularnewline
\midrule 
\multirow{2}{*}{Transformer$^*$} & MAE & 21.69 & 30.44 & 39.21 & 31.17 &21.30  &  27.58 & 42.91 & 30.02 &2.43  & 3.08 & \underline{3.45}  & 2.92 & 0.27 & 0.29 &0.31  & 0.29  \tabularnewline
 & RMSE  & 33.32 & 42.99 & 61.13 & 50.16 & 35.10 & 42.89 &  60.00 & 48.22 & 6.20 & 7.74 & \underline{8.46} & 7.39  & 0.64 & 0.69 & 0.75 & 0.69  \tabularnewline
\midrule 
\multirow{2}{*}{iTransformer} & MAE & 18.90 & 25.76 & 36.58 & 26.13 & 19.33 &25.64  & 35.89 & 26.00 & 2.62 & 3.82 & 5.87 & 3.91 & 0.28 & 0.32 & 0.35 & 0.31  \tabularnewline
 & RMSE  &30.94 & 41.49 & 57.74 & 43.35 & 32.00 & 41.02 &55.98  & 42.68 & 6.47 &9.65  & 14.35 & 10.29 & 0.73 &0.83  &  0.95 & 0.83  \tabularnewline
\midrule 
\cmidrule[0.5pt]{1-18} 
\multirow{2}{*}{\textbf{ScaleSTF (ours)}} & MAE & \textbf{15.56} & \textbf{18.50} & \textbf{22.43} & \textbf{18.38} & \textbf{16.23} & \textbf{18.81} & \textbf{22.10} & \textbf{18.59} & \textbf{2.03} & \textbf{2.75} & \textbf{3.35} & \textbf{2.60} & \textbf{0.24} & \textbf{0.26} & \textbf{0.28} & \textbf{0.25} \tabularnewline
 & RMSE & \textbf{25.99} & \textbf{31.10} & \textbf{38.24} & \textbf{31.43} & \textbf{27.86} & \textbf{31.85} & \textbf{37.04} & \textbf{31.81} & \textbf{5.52} & \textbf{7.21} & \textbf{8.35} & \textbf{6.92} & \textbf{0.60} & \textbf{0.64} & \textbf{0.67} & \textbf{0.62} \tabularnewline
\cmidrule{1-18} 
\multirow{2}{*}{\textbf{Avg. Imp.}$^\dagger$} & MAE & \multicolumn{4}{c|}{\textbf{10.81\%}}  & \multicolumn{4}{c|}{\textbf{9.54\%}}& \multicolumn{4}{c|}{\textbf{7.14\%}}& \multicolumn{4}{c}{\textbf{10.71\%}} \tabularnewline
 & RMSE & \multicolumn{4}{c|}{\textbf{6.68\%}} & \multicolumn{4}{c|}{\textbf{4.53\%}}& \multicolumn{4}{c|}{\textbf{4.81\%}}& \multicolumn{4}{c}{\textbf{6.06\%}} \tabularnewline
\bottomrule
\end{tabular}
}
\begin{tablenotes} 
\footnotesize
{\item $^*$: Note that the canonical spatial attention runs out of memory on these large-scale benchmarks, and we only adopt the temporal attention for the \texttt{Transformer}. $^\dagger$: The average performance gains over the second-best models.
}
\end{tablenotes} 
\end{table*}

\noindent\textbf{Baselines.} 
Due to the large scales of the adopted datasets, we carefully select several applicable and competitive baselines, and they include: DCRNN~\cite{li2017diffusion}, AGCRN~\cite{bai2020adaptive}, STGCN~\cite{yu2017spatio}, GWNet~\cite{wu2019graph}, TSMixer~\cite{chen2023tsmixer}, Transformer, and iTransformer~\cite{liu2023itransformer}.
Please note that many advanced models such as STAEformer \cite{STAEformer}, D2STGNN \cite{shao2022decoupled}, and PDFormer \cite{jiang2023pdformer} have shown SOTA performance in medium-sized datasets. However, \textbf{they fail to function across our {large-scale} benchmarks} due to the high computational complexity. {Therefore, we only apply them in medium-scale datasets in Section \ref{subsec:transformer_compare}.}

\noindent\textbf{Implementation and Hyperparameters.}
All models are implemented using the TorchSpatiotemporal benchmark tool on a single NVIDIA RTX A6000 GPU (48GB).
Hyperparameters of all models are tuned using cross-validation, and we will release them as well as the reproducible codes after publication.

\begin{table*}
\renewcommand{\arraystretch}{1.2} 
    \centering
    \setlength{\tabcolsep}{1.pt}
    \caption{{Comparison with SOTA Transformer-based model in GP-VAR(-L), AirQuality, and Elergone datasets.}}
    \label{tab:gpvar}
\resizebox{0.99\textwidth}{!}{
    \begin{tabular}{c|cccc|>{\columncolor{gray25}}c>{\columncolor{gray25}}c>{\columncolor{gray25}}c>{\columncolor{gray25}}c|cccc|>{\columncolor{gray25}}c>{\columncolor{gray25}}c>{\columncolor{gray25}}c>{\columncolor{gray25}}c}
    \toprule 
    \multicolumn{1}{c|}{\textbf{Dataset}} & \multicolumn{8}{c|}{\textbf{GPVAR-L (600 nodes)}} & \multicolumn{8}{c}{\textbf{GPVAR (600 nodes)}}  \tabularnewline
    \midrule 
      \multirow{2}{*}{Method}& \multicolumn{4}{c|}{\textbf{Prediction error (MAE)}} & \multicolumn{4}{c|}{\textbf{Resource utilization}}  & \multicolumn{4}{c|}{\textbf{Prediction error (MAE)}} & \multicolumn{4}{c}{\textbf{Resource utilization}} \tabularnewline
 & \textbf{@3} & \textbf{@6} & \textbf{@12} & \textbf{Avg.} & \textbf{Batch/s} & \textbf{Memory} & \textbf{Param.} & \textbf{Batch Size}  & \textbf{@3} & \textbf{@6} & \textbf{@12} & \textbf{Avg.} & \textbf{Batch/s} & \textbf{Memory}& \textbf{Param.} & \textbf{Batch Size}   \tabularnewline
 \midrule
 \multirow{1}{*}{{PDFormer}} & {.5990} & {.7209} & {.8516}  & {.7022}  & {1.88} & {15.1 GB} &  {3.90 M} & {16}  & {.3470}  & {.3501} & {.3528} & {.3492} & {8.69}  & {8.8 GB}  &  {1.30 M}  & {16}  \tabularnewline
 \midrule
 \multirow{1}{*}{STAEformer} & \underline{.5876} & \textbf{.7109} & \underline{.8333} & \textbf{.6882} &  \underline{3.61} & \underline{13.0 GB} & \underline{2.70 M} & 16  & \underline{.3405} & \underline{.3463} & \underline{.3472} & \underline{.3419} & \underline{12.66} & \underline{6.9 GB} &  \underline{0.78 M} & 16 \tabularnewline
\midrule 
\cmidrule[0.5pt]{1-17} 
\multirow{1}{*}{\textbf{ScaleSTF (ours)}}& \textbf{.5713} & \underline{.7127} & \textbf{.8325} & \underline{.6907} & \textbf{68.09} & \textbf{2.1 GB} & \textbf{0.82 M} & {16}  & \textbf{.3403} & \textbf{.3450} & \textbf{.3468} & \textbf{.3415} & \textbf{104.75} & \textbf{2.0 GB} & \textbf{0.10 M} & {16} \tabularnewline
\bottomrule
\toprule 
    \multicolumn{1}{c|}{\textbf{Dataset}} & \multicolumn{8}{c|}{\textbf{{AirQuality (437 nodes)}}} & \multicolumn{8}{c}{\textbf{{Elergone (370 nodes)}}}  \tabularnewline
    \midrule 
      \multirow{2}{*}{Method}& \multicolumn{4}{c|}{\textbf{Prediction error (MAE)}} & \multicolumn{4}{c|}{\textbf{Resource utilization}}  & \multicolumn{4}{c|}{\textbf{Prediction error (MAE)}} & \multicolumn{4}{c}{\textbf{Resource utilization}} \tabularnewline
 & \textbf{@1} & \textbf{@2} & \textbf{@3} & \textbf{Avg.} & \textbf{Batch/s} & \textbf{Memory} & \textbf{Param.} & \textbf{Batch Size}  & \textbf{@3} & \textbf{@6} & \textbf{@12} & \textbf{Avg.} & \textbf{Batch/s} & \textbf{Memory}& \textbf{Param.} & \textbf{Batch Size}   \tabularnewline
 \midrule
 \multirow{1}{*}{{PDFormer}} & {11.42} & {15.66} & {18.30}  & {22.14} & {0.97} & {25.6 GB} & {4.0 M} & {32}  & {210.00} & {228.48} & {231.57} & {220.61}& {1.54} &  {20.1 GB} &  {4.0 M} & {32} \tabularnewline
 \midrule
 \multirow{1}{*}{{STAEformer}} & {\textbf{11.05}} & {\textbf{14.83}} & {\textbf{17.56}}  &  {\underline{21.97}} & {\underline{1.93}} & {\underline{20.7 GB}} & {\underline{3.3 M}} & {32}  & {\underline{202.02}} & {\underline{224.60}} & {\underline{225.69}} & {\underline{215.94}} & {\underline{2.78}} & {\underline{15.8 GB}} &  {\underline{3.2 M}} & {32} \tabularnewline
\midrule 
\cmidrule[0.5pt]{1-17} 
\multirow{1}{*}{\textbf{{ScaleSTF (ours)}}}& {\underline{11.19}} & {\underline{15.10}} & {\underline{17.90}} & {\textbf{21.93}} & {\textbf{50.35}} & {\textbf{1.8 GB}} & {\textbf{1.0 M}} & {32}  & {\textbf{199.10}} & {\textbf{222.41}} & {\textbf{208.62}} & {\textbf{208.83}} & {\textbf{47.09}} & {\textbf{1.8 GB}} & {\textbf{0.77 M}} & {32} \tabularnewline
\bottomrule
\end{tabular}
}
\end{table*}

\subsection{Performance Comparison in Short-term Benchmarks}
We first evaluate the performances to predict short-term dynamics. For all datasets, we set both the look-back window and the prediction horizon to 12 steps and report the error metrics.
Results of the model comparisons are shown in Table \ref{tab:main}. Generally, ScaleSTF consistently achieves SOTA performance in all metrics in all tasks.
Notably, compared to GNN- and Mixer-based models, ScaleSTF improves accuracy by a large margin, demonstrating the effectiveness of the Transformer-like architecture in learning graph representations. The comparison between ScaleSTF, Transformer, and iTransformer also justifies our physical inductive bias for large-scale STGs.

\begin{figure}[!htbp]
  \centering
  \includegraphics[width=1\columnwidth]{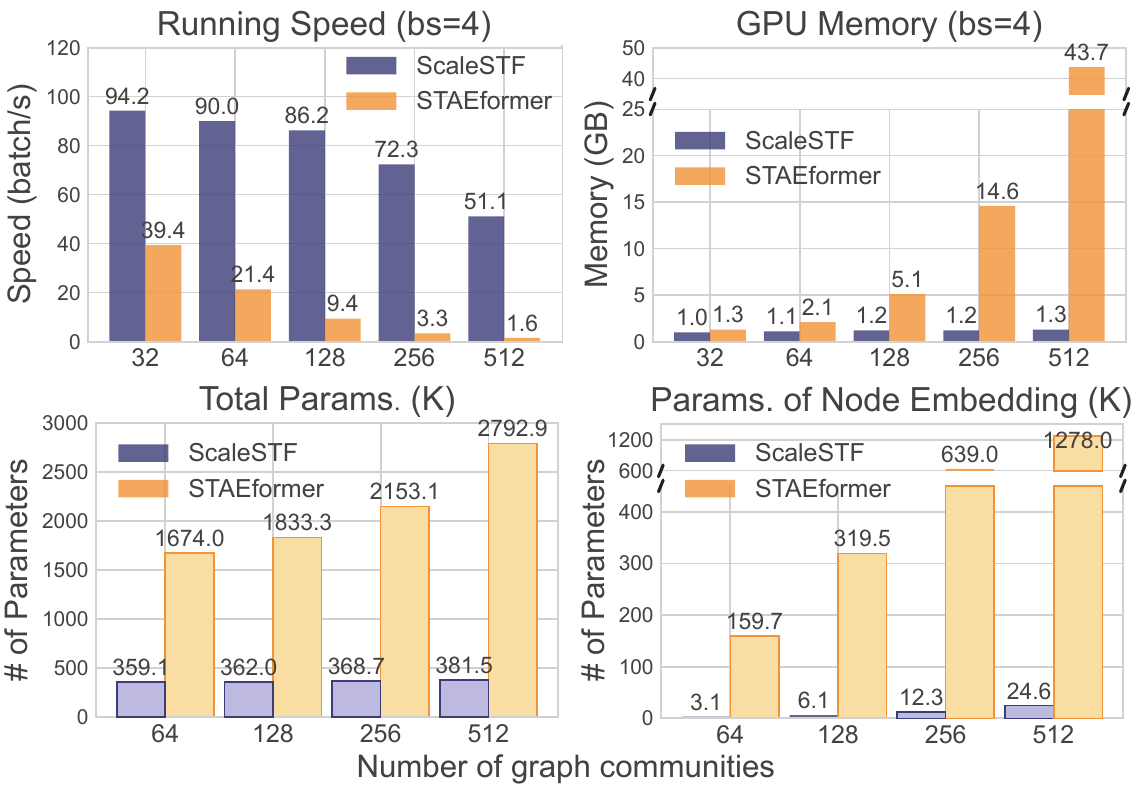}
  \caption{Model scalability with varying number of nodes (batch size = 4).}
  \label{fig:scalability}
\end{figure}

\subsection{Comparison with SOTA Transformers}\label{subsec:transformer_compare}
Next, we compare ScaleSTF with a SOTA STF model, STAEformer \cite{STAEformer}. Since STAEformer cannot work on the four large datasets used above with resource limitation, we adopted GP-VAR data to control the scale of generated graphs. {Two real-world medium-scale datasets including AirQuality and Elergone are also adopted to benchmark the performances}.

\noindent\textbf{Overall Performance.} 
Tab. \ref{tab:gpvar} shows overall performances of two models in accuracy and efficiency. ScaleSTF performs comparably with STAEformer in terms of accuracy, but it shows great superiority in computational efficiency and resource preservation. In particular, ScaleSTF provides up to \textbf{18x speed-up}, \textbf{7x memory reduction}, and \textbf{6x parameter savings} over the SOTA, indicating great potential for large networks.

\noindent\textbf{Scalability and Parameter Efficiency.} We further examine the scalability under varying numbers of nodes. 
A larger number of graph communities have more nodes.
Fig. \ref{fig:scalability} denotes that ScaleSTF shows surprisingly desirable scalability. Instead, STAEformer runs out of memory on graphs with thousands of nodes and has a slow running speed. In addition, low-rank designs significantly reduce model parameters, leading to an efficient architecture with much fewer parameters to optimize.

\subsection{Predicting Long-term Network Dynamics}
In addition to the short-term prediction, we also evaluate the long-term prediction performance.
All models are trained to predict next 192 steps with a historical window of 96 steps. Tab. \ref{tab:results_long} shows the results of model comparison. ScaleSTF shows great superiority in accuracy and has a low memory consumption comparable to MLPs. However, many complex models fail to complete these tasks with limited resources.

\begin{table}
\renewcommand{\arraystretch}{1.1} 
    \centering
    \setlength{\tabcolsep}{1.pt}
    \caption{{Results of long-term dynamics prediction (batch size is 4).}}
    \label{tab:results_long}
\resizebox{0.99\columnwidth}{!}{
    \begin{tabular}{c|cc|cc|cc|cc}
    \toprule 
    \multicolumn{1}{c|}{\textbf{Dataset}} & \multicolumn{4}{c|}{\textbf{GLA}} & \multicolumn{4}{c}{\textbf{GBA}}  \tabularnewline
    \midrule 
    \textbf{Method} & \textbf{MAE} & \textbf{RMSE} & \textbf{Memory} & \textbf{Speed} & \textbf{MAE} & \textbf{RMSE} & \textbf{Memory} & \textbf{Speed}\tabularnewline
    \midrule 
DCRNN &  \multicolumn{2}{c|}{OOT} & {27.3 GB} & {$<0.1$ B/s} &  \multicolumn{2}{c|}{OOT} & {20.1 GB} & {$<0.1$ B/s} \tabularnewline
\midrule 
AGCRN & \multicolumn{2}{c|}{OOM} & {$>$ 48 GB} & {--} & \multicolumn{2}{c|}{OOM} & {$>$ 48 GB} & {--} \tabularnewline
\midrule 
STGCN & 30.12 & 48.28 & 18.6 GB & {1.03 B/s} & 32.66 & 50.09 & 13.2 GB  & {1.67 B/s} \tabularnewline
\midrule 
GWNet & 28.70 & 46.92 & 27.9 GB & {0.60 B/s} & 29.77 & \underline{48.26} & 19.3 GB & {1.02 B/s} \tabularnewline
\midrule 
TSMixer & 29.61 & \underline{46.43} & \textbf{1.9 GB} & {\textbf{30.63 B/s}} & 30.87 & 51.78 & \textbf{1.7 GB} & {\textbf{47.15 B/s}} \tabularnewline
\midrule 
Transformer$^*$ & \multicolumn{2}{c|}{OOM} & {$>$ 48 GB} & {--} & \multicolumn{2}{c|}{OOM} & {$>$ 48 GB} & {--} \tabularnewline
\midrule 
iTransformer & \underline{28.35} & 47.81 & 13.9 GB & {4.37 B/s}  &  \underline{29.15} & 48.50 & 6.2 GB & {9.55 B/s} \tabularnewline
\midrule 
\cmidrule[0.5pt]{1-9} 
\textbf{ScaleSTF (ours)}  & \textbf{23.86} & \textbf{39.89} & \underline{2.5 GB} & {\underline{16.54 B/s}} & \textbf{24.36} & \textbf{42.77} & \underline{2.0 GB} & {\underline{25.39 B/s}} \tabularnewline
\cmidrule{1-9} 
\multirow{2}{*}{{\textbf{Avg. Imp.}$^\dagger$}} & \cellcolor{gray25}{MAE} & \multicolumn{3}{c|}{\cellcolor{gray25}{\textbf{15.84 \%}}} & \cellcolor{gray25}{MAE}  & \multicolumn{3}{c}{\cellcolor{gray25}{\textbf{16.43 \%}}}
\tabularnewline
 & \cellcolor{gray25}{RMSE} & \multicolumn{3}{c|}{\cellcolor{gray25}{\textbf{14.09 \%}}} & \cellcolor{gray25}{RMSE} & \multicolumn{3}{c}{\cellcolor{gray25}{\textbf{11.38 \%}}}
\tabularnewline
\bottomrule
\end{tabular}
}
\begin{tablenotes} 
\footnotesize
{\item OOT indicates that the training cannot be finished within an acceptable time budget, and OOM indicates out-of-memory.
}
\end{tablenotes} 
\end{table}

\begin{figure}[!htbp]
  \centering
  \includegraphics[width=0.95\columnwidth]{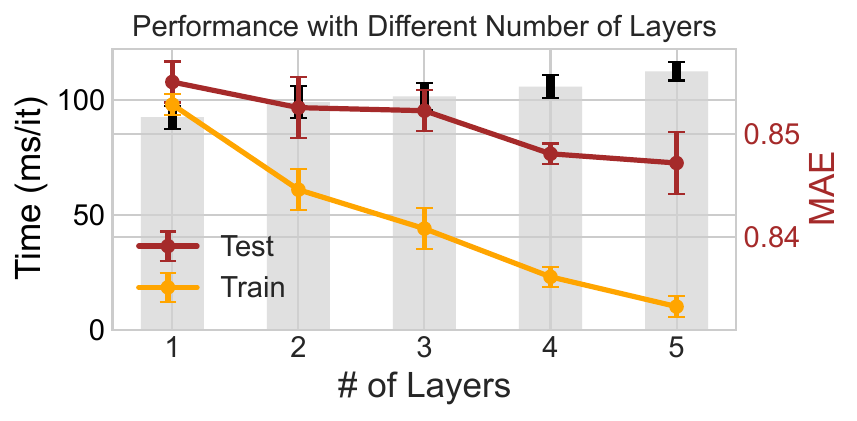}
  \caption{{Performance with different number of layers.}}
  \label{fig:error_bar}
\end{figure}

\subsection{Model Analysis}
\noindent\textbf{Study on the Model Depth.} 
Fig. \ref{fig:error_bar} plots the training time cost per iteration and {training/testing MAE} w.r.t. the number of neural diffusion layers. As observed, increasing the model depth can reduce {both training and testing} errors due to the layerwise denoising effect. But it does not significantly increase the computation time.
{In addition, while increasing the number of layers consistently reduces the training error, the improvement in test performance begins to slow down after a certain depth threshold. This suggests current dataset size (600 nodes) may be insufficient to support such increased capacity, resulting in reduced generalization. 
Furthermore, this phenomenon highlights the expressive power of our proposed model. The fact that deeper architectures can overfit implies that the model has sufficient capacity to accurately approximate the training distribution. With appropriate data scaling, its generalizability can be further improved. }

\noindent\textbf{Redundancy in ST-Transformers.} 
To justify our hypothesis, we illustrate the architectural redundancy in STFs in Fig. \ref{fig:redundancy}. The proposed LRAE shares a nuclear norm similar to the method in \cite{STAEformer}, but has a markedly reduced effective rank, alleviating the overparameterization issue in node embedding. In addition, our modulated attention can concentrate on dominant node patterns with a few large singular values (the 1st singular value is omitted for clearer visualization), reducing the redundancy in the self-attention matrix.

\begin{figure}[!htbp]
  \centering
  \includegraphics[width=1\columnwidth]{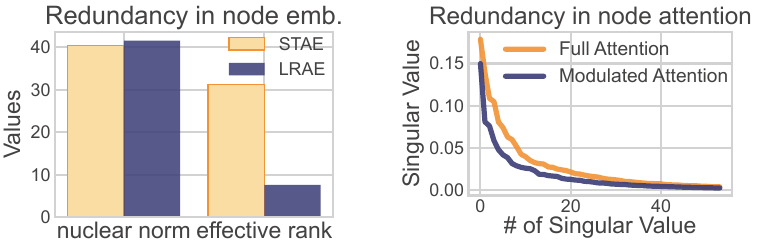}
  \caption{Examples of redundancy in ST-Transformers.}
  \label{fig:redundancy}
\end{figure}

\noindent\textbf{Visualization of the Learned Embedding.} Fig. \ref{fig:emb} shows the t-SNE visualization of the learned node embedding in GBA data. The dimension-reduced manifold clearly shows several clusters, revealing the existence of low-dimensional structures.

\begin{figure}[!htbp]
  \centering
  \includegraphics[width=0.8\columnwidth]{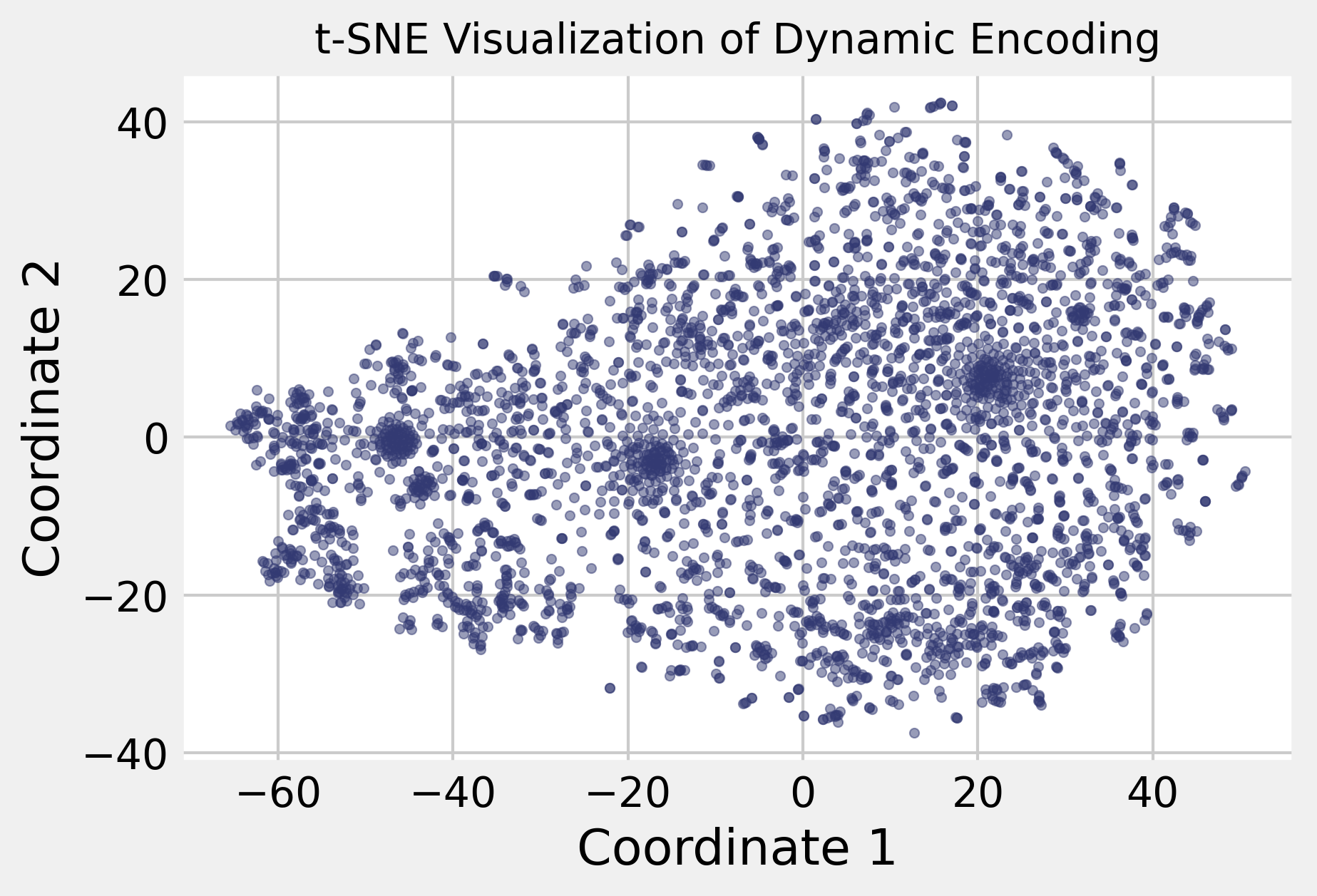}
  \caption{The t-SNE structure of the latent node embedding.}
  \label{fig:emb}
\end{figure}

{
\noindent\textbf{Prediction with Sparse/Noisy Observations.} }
Since our model is established by modeling a graph denoising process, it can naturally deal with missing data in the observation. We randomly mask out $80\%$ of the observations and train the model. Fig. \ref{fig:prediction} shows that our model can still produce desirable predictions even for intervals with very few data.
This result shows its potential for the spatiotemporal imputation task \cite{nie2024imputeformer}.

{We further quantitatively compare the robustness of models under different missing ratios (Fig. \ref{fig:missing-data}) and different levels of noise (Fig. \ref{fig:motivation} (b)). It is observed that due to the structured per-layer denoising (diffusion) process, our model shows better robustness than other non-structured models.}

\begin{figure}[!htbp]
  \centering
  \includegraphics[width=1\columnwidth]{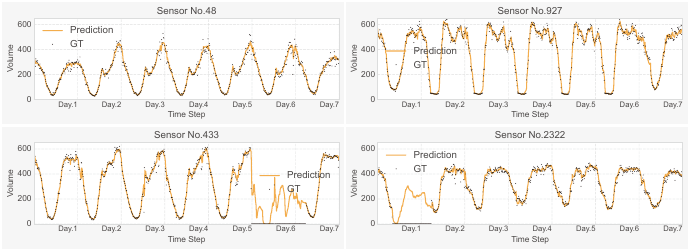}
  \caption{Prediction results with missing values.}
  \label{fig:prediction}
\end{figure}

\begin{figure}[!htbp]
  \centering
  \includegraphics[width=0.7\columnwidth]{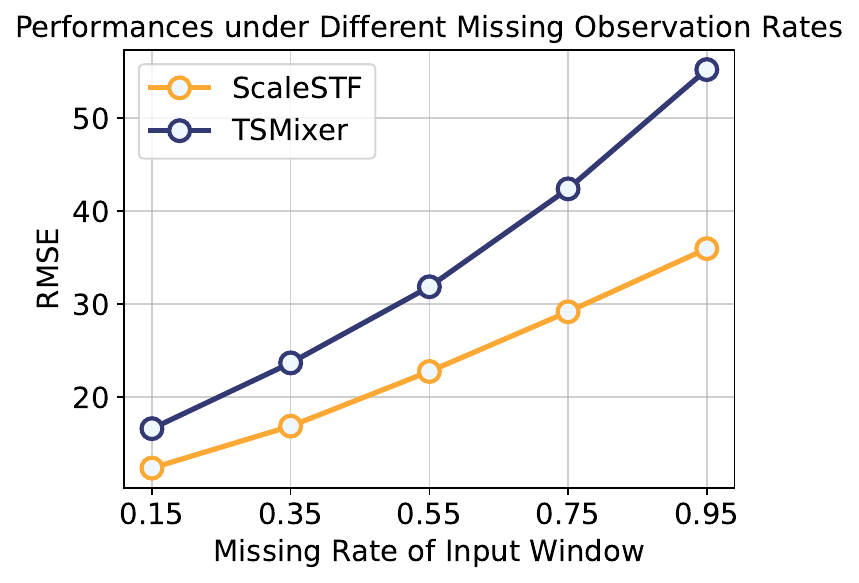}
  \caption{{Performance comparison under different missing values.}}
  \label{fig:missing-data}
\end{figure}

\begin{table}
\renewcommand{\arraystretch}{1.1} 
    \centering
    \setlength{\tabcolsep}{1.pt}
    \caption{{Ablation Studies (Long-term Prediction).}}
    \label{tab:ablation}
\resizebox{0.99\columnwidth}{!}{
    \begin{tabular}{c|cccc|cccc}
    \toprule 
    \multicolumn{1}{c|}{\textbf{Dataset}} & \multicolumn{4}{c|}{\textbf{GLA}} & \multicolumn{4}{c}{\textbf{GBA}}  \tabularnewline
    \midrule 
    \textbf{Method} & \textbf{MAE} & \textbf{Param.} & \textbf{Memory} & \textbf{Speed} & \textbf{MAE} & \textbf{Param.} & \textbf{Memory} & \textbf{Speed}\tabularnewline
\midrule 
w/ canonical attention & 25.23 & 2.0 M & 6.6 GB & 9.92 B/s & 25.42 & 2.0 M & 5.5 GB & 16.38 B/s \tabularnewline
\midrule 
w/o LRAE & 26.85 & 1.1 M & 2.5 GB & 16.60 B/s & 25.75 & 1.1 M & 2.0 GB & 25.58 B/s \tabularnewline
w/ dense embedding & 24.77 & 20.9 M & 2.6 GB & 15.82 B/s & 25.39 & 20.7 M & 2.1 GB & 24.90 B/s \tabularnewline
\midrule 
\textbf{ScaleSTF}  & {23.86} & {2.2 M} & {2.5 GB} & 16.54 B/s & {24.36} & {2.2 M} & {2.0 GB} &  25.39 B/s \tabularnewline
\bottomrule
\end{tabular}
}
\end{table}

\noindent{\textbf{Ablation Studies.}
To justify the modular designs, we perform ablation studies on long-term tasks. There are several findings in Tab. \ref{tab:ablation}: (1) Our proposed modulated node attention not only improves the efficiency of canonical self-attention, but also reduces prediction errors by resolving the redundancy in pairwise diffusivity; (2) LRAE plays a key role in reducing trainable parameters and helping in modeling node dynamics.
}

\noindent\textbf{{Impact of the Low-rank Embedding.} }
{The low-rank factorization is the cornerstone of our model. Fig. \ref{fig:rank} studies the impact of the rank parameter in Eq. \eqref{eq:lare}. As can be seen, the rank value controls both the accuracy and the number of parameters. A proper value (e.g. 16 in this case) can balance both effectiveness and efficiency.}

\begin{figure}[!htbp]
  \centering
  \includegraphics[width=0.95\columnwidth]{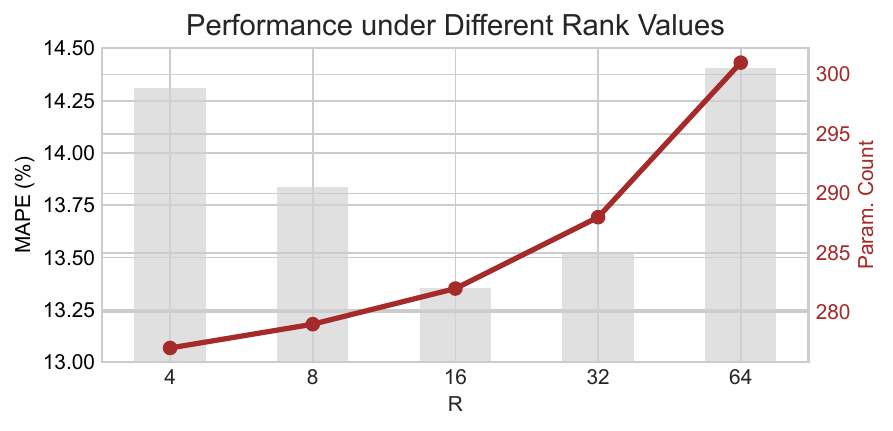}
  \caption{{Performance with different rank values in Eq. \eqref{eq:lare}.}}
  \label{fig:rank}
\end{figure}

{
\subsection{Case Study}
To further enhance interpretability in real-world scenarios, we provide a case study using traffic flow data on the California road network. Recall that the LRAE reflects the coordinate of each node in the embedding space, it can mirror the pattern similarity in the physical space. In Fig. \ref{fig:case-study} (a), we select several sensors with high feature similarities according to the pairwise similarity matrix in (b). We then show the flow profiles of these sensors in (c). It is observed that these sensors encounter the same traffic congestion during this time period, with a clear delay propagation path from sensor \# 945 to \# 905.

\begin{figure}[!htbp]
  \centering
  \includegraphics[width=0.99\columnwidth]{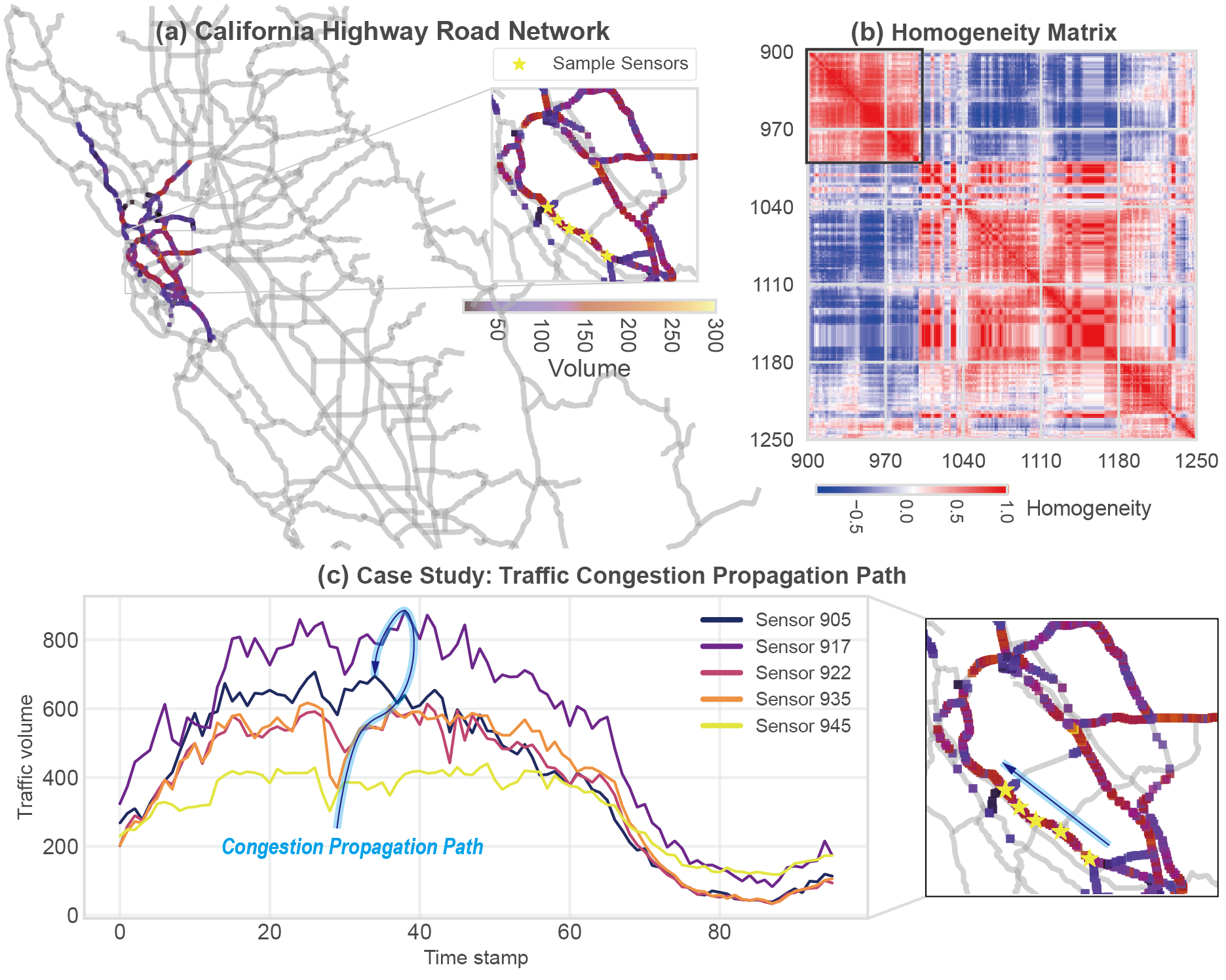}
  \caption{{Case study using GBA traffic flow data.}}
  \label{fig:case-study}
\end{figure}

}

\section{Conclusion}\label{sec:conclusion}
This paper links the neural diffusion process and the graph denoising problem to predict the dynamics of large-scale urban networks. Based on the theoretical analysis, we present a scalable spatiotemporal Transformer model, called ScaleSTF, to balance both effectiveness and efficiency.
With linear complexity, ScaleSTF achieves SOTA performance on large-scale benchmarks with a much reduced computational burden.
It can yield up to more than 10x speed acceleration over SOTA Transformers, significantly reducing parameters and memory usage.
Beyond the current results, we believe that the proposed methodology can facilitate the build of foundational Transformers on large networked urban systems.

{
Future efforts can investigate the computational complexity and scalability of ScaleSTF in real-time systems, especially those operating in high-frequency data streams, such as meteorological monitoring or smart grid environments. Although our current implementation demonstrates promising performance, deploying the model in latency-sensitive settings requires optimizing inference efficiency and ensuring that model updates can be performed incrementally or in an online fashion.
}

\bibliographystyle{IEEEtran}
\bibliography{references}


\begin{IEEEbiography}[{\includegraphics[width=1in,height=1.25in,keepaspectratio]{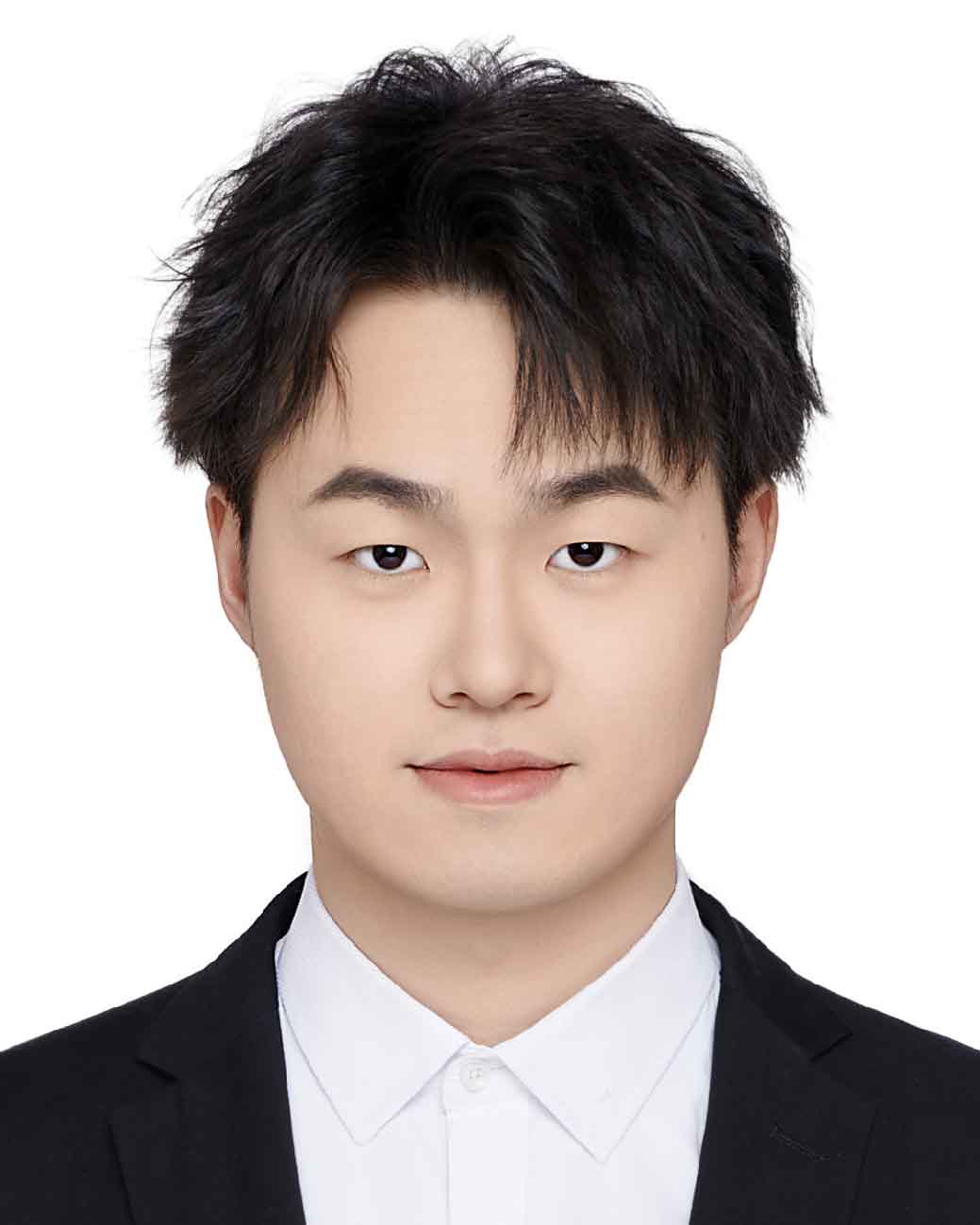}}]{Tong Nie} received the B.S. degree from the college of civil engineering, Tongji University, Shanghai, China. He is currently pursuing dual Ph.D. degrees with Tongji University and The Hong Kong Polytechnic University. He has published several papers in top-tier venues in the field of spatiotemporal data modeling, including KDD, AAAI, CIKM, IEEE TITS, and TR-Part C/E.
His research interests include spatiotemporal learning, time series analysis, and large language models. His research is funded by the National Natural Science Foundation of China.
\end{IEEEbiography}


\begin{IEEEbiography}[{\includegraphics[width=1in,height=1.25in,keepaspectratio]{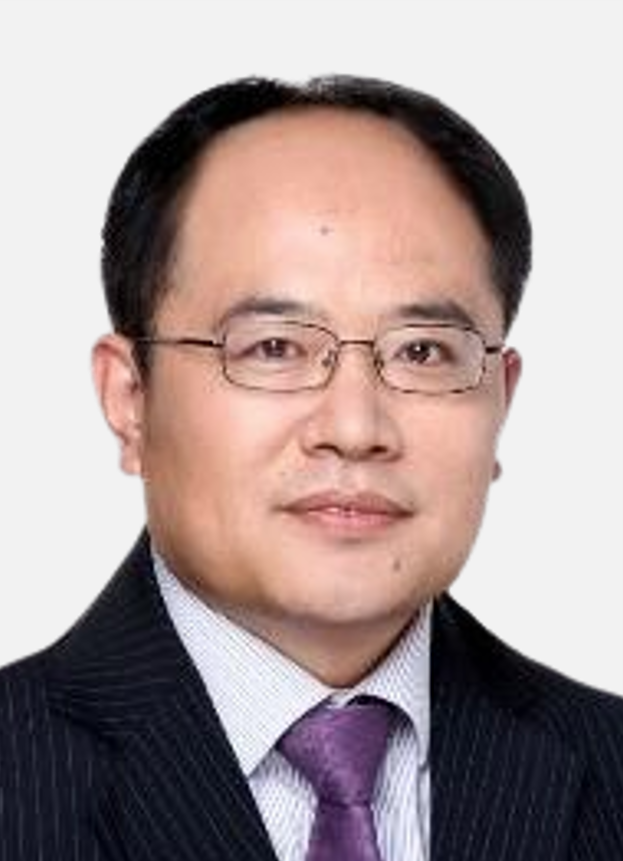}}]{Jian Sun (Senior Member, IEEE)} received the Ph.D. degree in transportation engineering from Tongji University, Shanghai, China. He is currently a Professor of transportation engineering with Tongji University. He has published more than 200 papers in SCI journals.
His research interests include intelligent transportation systems, traffic flow theory, AI in transportation, and traffic simulation. 
\end{IEEEbiography}

\begin{IEEEbiography}
[{\includegraphics[width=1in,height=1.25in,keepaspectratio]{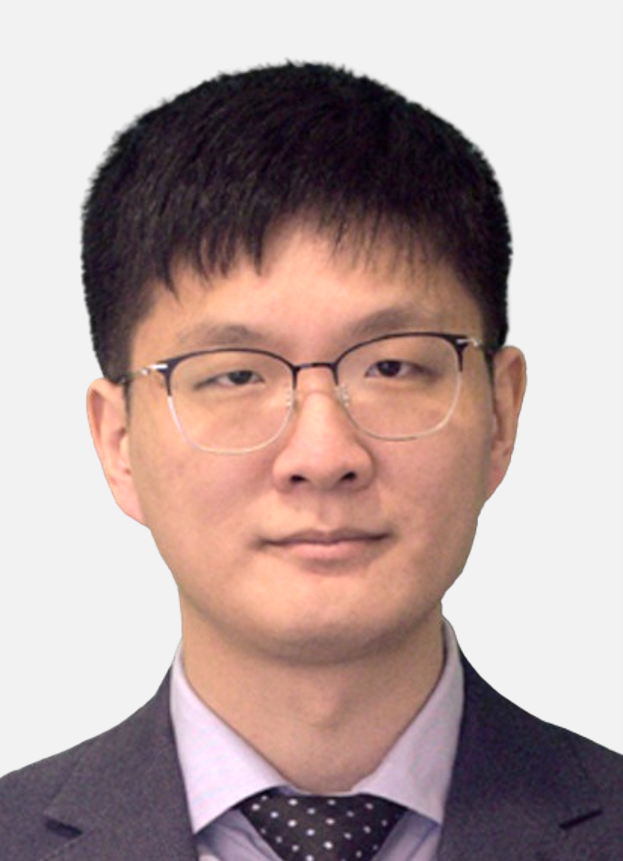}}]{Wei Ma (Member, IEEE)} received the bachelor’s
degree in civil engineering and mathematics from
Tsinghua University, China, and the master’s degree
in machine learning and civil and environmental engineering and the Ph.D. degree in civil and
environmental engineering from Carnegie Mellon
University, USA. He is currently an Assistant Professor with the Department of Civil and Environmental
Engineering, The Hong Kong Polytechnic University (PolyU). His current research interests include
machine learning, data mining, and transportation
network modeling, with applications for smart and sustainable mobility
systems.
\end{IEEEbiography}
\vfill

\end{document}